%% file: main.tex
\documentclass[lettersize,journal]{IEEEtran}

\input{defs}

\title{Biconvex Optimization for Smooth Minimum-Time Trajectories around Convex Obstacles}

\author{Peter Werner\textsuperscript{1}, Tobia Marcucci\textsuperscript{2}, and Daniela Rus\textsuperscript{1}
\thanks{\textsuperscript{1}Massachusetts Institute of Technology Computer Science and Artificial Intelligence Laboratory \quad \textsuperscript{2}University of California Santa Barbara}
}
\date{July 2026}

\IEEEoverridecommandlockouts
\begin{document}
\maketitle
\begin{abstract}
We present a biconvex approach for minimum-time motion planning around convex obstacles that is guaranteed to converge, is anytime, and supports derivative constraints to arbitrary order. We jointly convexify the minimum-time objective and all derivative constraints through a change of variables, and handle collision avoidance via time-varying separating planes, reducing the problem to a biconvex program. This program is solved by alternating between computing maximum-margin separating planes and optimizing the trajectory. By only adding planes for obstacles that the current iterate collides with, the trajectory can jump around obstacles and escape local minima. The method is guaranteed to converge starting from a simple collision-free polygonal curve. In our experiments on drone navigation and dual-arm bin unloading, we find that the proposed method reliably produces high-quality trajectories with computation times comparable to state-of-the-art decomposition-based motion planners, while handling a larger class of problems and being substantially more robust to bad initialization. Project page: \href{https://wernerpe.github.io/bmtp-website/}{wernerpe.github.io/bmtp-website/}.
\end{abstract}

\input{sections/introduction}
\input{sections/problem_statement}
\input{sections/convexity_background}
\input{sections/minimum_time}

\input{sections/collision_avoidance}
\input{sections/sequential_convex_programming}

\input{sections/discretization}

\input{sections/experiments}

\input{sections/discussion.tex}
\input{sections/software.tex}
\input{sections/conclusion.tex}

\input{sections/acknowledgements.tex}

\bibliographystyle{IEEEtran}
\bibliography{biblio}
\appendix
\input{sections/appendix}
\end{document}

%% file: defs.tex
\usepackage{graphicx,psfrag,amsmath,amsfonts,verbatim,tikz}
\usepackage{amsthm}
\usepackage{booktabs}
\usepackage{xcolor}
\usepackage{hyperref}
\usepackage{tikz}
\usepackage{multicol}
\usepackage{listings}
\usepackage{bbm}
\usepackage{fancyvrb}
\usepackage[most]{tcolorbox}
\usepackage{multirow}
\usepackage[ruled,vlined]{algorithm2e}
\usepackage{algorithmicx}
\usepackage{amssymb}%
\usepackage{stmaryrd}
\usepackage{pifont}
\newcommand{\inR}{\in \mathbb{R}}
\newcommand{\R}{\mathbb{R}}

\newcommand{\calC}{\ensuremath{\mathcal{C}}}

\newcommand{\calH}{\ensuremath{\mathcal{H}}}
\newcommand{\calI}{\ensuremath{\mathcal{I}}}

\newcommand{\calK}{\ensuremath{\mathcal{K}}}
\newcommand{\calL}{\ensuremath{\mathcal{L}}}

\newcommand{\calO}{\ensuremath{\mathcal{O}}}
\newcommand{\calP}{\ensuremath{\mathcal{P}}}

\newcommand{\calS}{\ensuremath{\mathcal{S}}}
\newcommand{\calT}{\ensuremath{\mathcal{T}}}

\newcommand{\conv}{\mathop{\bf conv}}

\usepackage[utf8]{inputenc}

\DeclareFixedFont{\ttb}{T1}{txtt}{bx}{n}{10} 
\DeclareFixedFont{\ttm}{T1}{txtt}{m}{n}{10}  

\usepackage{color}
\definecolor{deepblue}{rgb}{0,0,0.5}
\definecolor{deepred}{rgb}{0.6,0,0}
\definecolor{deepgreen}{rgb}{0,0.5,0}

\usepackage{listings}

\definecolor{codegreen}{rgb}{0,0.6,0}
\definecolor{codegray}{rgb}{0.5,0.5,0.5}
\definecolor{codepurple}{rgb}{0.58,0,0.82}
\definecolor{backcolour}{rgb}{0.95,0.95,0.92}

\lstdefinestyle{mystyle}{
	backgroundcolor=\color{backcolour},   commentstyle=\color{codegreen},
	keywordstyle=\color{magenta},
	numberstyle=\tiny\color{codegray},
	stringstyle=\color{codepurple},
	basicstyle=\ttfamily\footnotesize,
	breakatwhitespace=false,         
	breaklines=true,                 
	captionpos=b,                    
	keepspaces=true,                 
	numbers=left,                    
	numbersep=5pt,                  
	showspaces=false,                
	showstringspaces=false,
	showtabs=false,                  
	tabsize=2
}
\usepackage[abbreviations]{glossaries-extra}
\newabbreviation{gcs}{GCSTrajOpt}{\emph{Graph of Convex Sets}}

\newcommand{\bez}{Bézier\xspace}

\newtheorem{lemma}{Lemma}
\newtheorem{definition}{Definition}
\newtheorem{property}{Property}

\usepackage{cleveref}
\crefname{section}{sec.}{secs.}
\Crefname{section}{Sec.}{Secs.}
\crefname{figure}{fig.}{figs.}
\Crefname{figure}{Fig.}{Figs.}
\crefname{algorithm}{alg.}{algs.}
\Crefname{algorithm}{Alg.}{Algs.}
\crefname{table}{tab.}{tabs.}
\Crefname{table}{Tab.}{Tabs.}
\crefname{subappendix}{Appendix}{Appendices}
\Crefname{subappendix}{Appendix}{Appendices}

\crefname{algorithm}{alg.}{algs.}
\Crefname{algorithm}{Alg.}{Algs.}

\Crefname{lemma}{Lemma}{Lemmas}
\crefname{lemma}{Lemma}{Lemmas}
\Crefname{definition}{Definition}{Definitions}
\crefname{definition}{Definition}{Definitions}
\Crefname{property}{Property}{Properties}
\crefname{property}{Property}{Properties}

%% file: sections/introduction.tex
\section{Introduction}\label{sec:intro}

\IEEEPARstart{E}{fficiently} moving a robot between a starting position and a goal position, while not colliding with the environment, is a fundamental challenge in robotics. In high-throughput industrial settings such as manufacturing and warehouse logistics, robots must execute tasks reliably under tight time constraints, making both the quality and predictability of motion plans critical. Despite decades of research, the motion planning pipelines commonly deployed in such settings are often still simple heuristics, such as waypoint planners or precomputed motion primitives, that sacrifice performance in favor of reliability and predictable runtimes~\cite{ correll2018analysis, hernandez2016team, marcucci2025biconvex}.

We can observe a similar inefficiency in modern open-ended robotic manipulation. Recent advances in imitation learning have enabled robots to accomplish tasks that seemed out of reach just a few years ago~\cite{chi2025diffusion, intelligence2025pi, barreiros2025careful}. Yet the resulting policies tend to move slowly and give obstacles a wide berth. This is not out of necessity, but an artifact of how training data is collected.
Teleoperators prioritize task success and steer clear of obstacles, trading safety for efficiency. 

In both cases a drop-in planner that keeps this reliability while producing high-quality trajectories could improve throughput by either directly controlling the robots or polishing demonstrations post hoc.
\begin{figure}[t]
    \centering
    \includegraphics[width=\columnwidth]{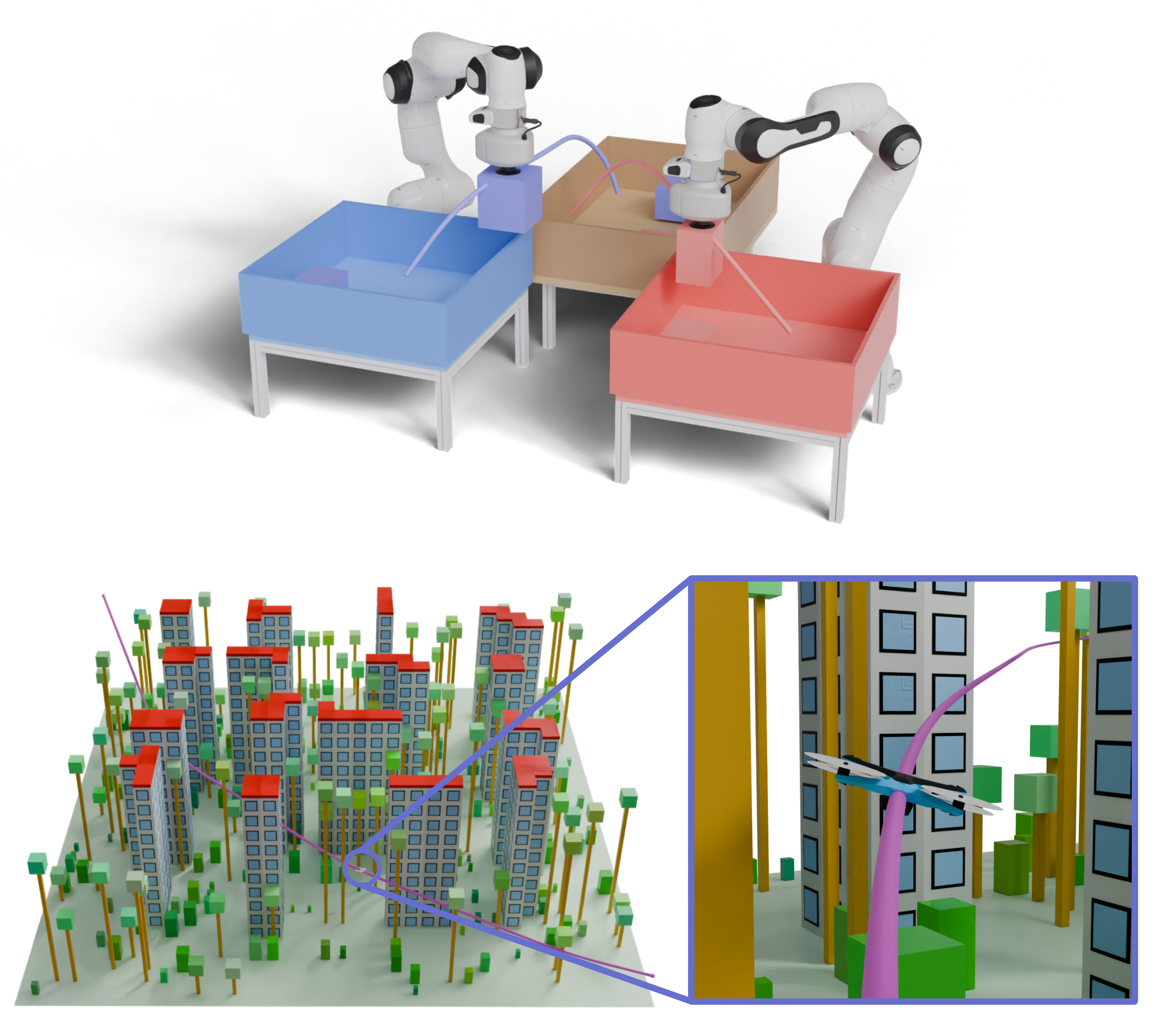}
    \caption{Our biconvex minimum-time planner computes smooth trajectories around convex obstacles. \textit{(top)} Two Franka arms cooperatively unload packages from a central bin to their respective offload bins using the planner. \textit{(bottom)} We employ the planner to compute a minimum-time trajectory with continuous snap and constrained velocity, acceleration, and jerk to navigate a drone through a village.}
    \label{fig:title}
\end{figure}

This paper proposes a biconvex minimum-time planner that is guaranteed to converge, is fast, produces high-quality trajectories, and supports derivative constraints to any degree, e.g., jerk, snap, and can be applied to tasks like bin unloading or drone navigation as shown in~\Cref{fig:title}.

\subsection{Established motion planning approaches}
For general applications, the established motion planning approaches can be broadly grouped into two categories: sampling-based planners and trajectory optimization.
Sampling-based planners such as RRTs and PRMs~\cite{kuffner2000rrt, kavraki1996probabilistic} are simple to implement, enjoy asymptotic completeness guarantees~\cite[\S 5]{lavalle2006planning}, and can be significantly accelerated through parallelization~\cite{pan2012gpu, thomason2024motions,huang2025prrtc}. However, they tend to produce jagged, circuitous paths and do not natively encode smoothness, velocity, or acceleration constraints. Extensions make these paths inherently smooth~\cite{webb2013kinodynamic} or improve solution quality with more samples, achieving asymptotic~\cite{karaman2011sampling} or almost-sure asymptotic optimality~\cite{wilson2025aorrtc}. Often practitioners also rely on post-processing of the found plans to improve their quality via shortcutting~\cite{geraerts2007creating, hauser2010fast} or B-spline smoothing~\cite{pan2012collision}. Additionally, optimal retiming~\cite{pham2018new} can enforce derivative constraints. Unfortunately, the listed steps are heuristic, can fail, or produce suboptimal results and generally add complexity to the motion planning stack which often makes it brittle.

Alternatively one can try to solve the entire problem in one shot by using local methods to directly optimize a trajectory~\cite{schulman2014motion, ratliff2009chomp}. There exist also nonlinear trajectory optimization methods that are tailored to separating convex obstacles from trajectories such as~\cite{zhang2020optimization, tordesillas2021mader}. Such optimization methods offer a natural way of encoding general costs and constraints, scale to high-dimensional spaces, and can be parallelized to check many initial guesses on modern GPUs~\cite{sundaralingam2023curobo}.
The main challenge in using these methods is dealing with nonconvex constraints such as collision-avoidance constraints, which can cause local solvers to get stuck or fail, produce inconsistent runtimes, and often require careful warm-starting. These issues make them difficult to deploy reliably in practice.

\subsection{Decomposition-based motion planners}\label{ssec:dbmp_overview}
A recent line of work on DBMPs~\cite{deits2015efficient, marcucci2024shortest, marcucci2023motion, marcucci2024fast, werner2025superfast, marcucci2025biconvex, chen2016online, liu2017planning, wu2024optimal, natarajan2024implicit,chia2024gcs,morozov2025mixed, yang2025new} attempts to combine the strengths of both approaches: reliable performance and richer problem descriptions in terms of costs and constraints. These approaches show promise in demanding warehouse settings~\cite{marcucci2025biconvex}. The key idea is to (approximately) decompose the collision-free space into a union of convex sets and restrict the trajectory to lie within it. The benefit is that challenging nonconvex collision-avoidance constraints are replaced with simple convex constraints per-set and the nonconvexity is pushed into selecting which sets the trajectory visits.


Further, technical challenges arise when attempting to optimize the trajectory duration while enforcing derivative constraints beyond the first derivative, such as requiring a trajectory to have bounded jerk or obey acceleration limits. These constraints are often crucial in practice. We briefly explain the nature of these technical challenges in~\Cref{app:dbmp_nonconvexity}. To deal with these challenges, the authors fix the trajectory duration~\cite{deits2015efficient, chen2016online, liu2017planning, wu2024optimal}, devise a sequential optimization approach~\cite{marcucci2024fast,marcucci2025biconvex}, support only a subset of these constraints~\cite{marcucci2024graphs,natarajan2024implicit,chia2024gcs}, or discretize the trajectory duration~\cite{morozov2025mixed}. Alternatively, the authors in~\cite{yang2025new} represent the trajectory with a discrete set of points and a variable time scaling, and devise a semidefinite relaxation and rounding scheme for the resulting nonconvex trajectory optimization problem. However, this approach provides no guarantees on the tightness of the relaxation or the robustness of the rounding procedure. None of the referenced DBMPs can enforce both higher-order derivative constraints and continuity to arbitrary order in a convex fashion. 

Beyond these technical challenges, DBMPs additionally require a description of the free space as a union of convex sets in order to be deployed. These descriptions can be computed with polytope inflation algorithms such as~\cite{deits2015computing,liu2017planning, werner2024approximating,werner2024faster, dai2024certified}. A key challenge is that this step tends to be computationally costly and needs to be performed every time the environment changes. Recently, it was significantly accelerated with GPU parallelization~\cite{werner2025superfast}, but the cost still remains substantial. 

\subsection{Proposed biconvex minimum-time planner}
In this paper, we propose a biconvex minimum-time planner (BMTP) that attempts to tackle all of the above listed challenges: it supports derivative constraints beyond first- and second-order, it supports continuity constraints on the trajectory to an arbitrary finite degree, and does not require a costly approximate convex decomposition of the free space up front. In our approach, we apply a similar change of variables as in~\cite{leomanni2022time} to convexify the minimum-time objective and derivative constraints up to an arbitrary degree. We then formulate the collision-avoidance constraints as an intersection of time-varying halfspaces, yielding a time-varying, safe polytope. Concretely, for every time between the start and the end of the trajectory, the evaluated halfspaces separate an obstacle from the corresponding point on the trajectory. This allows trajectories to curve around obstacles without the explicit need for multiple segments as in prior DBMP works. Finally, BMTP is an effective anytime procedure that refines an initial collision-free trajectory by alternating between computing a valid sequence of convex sets around it and updating the trajectory until convergence. Notably, our approach can be initialized with a simple polygonal curve, which can be hand-crafted, or generated with any sampling-based motion planner. We also provide an open-source implementation of our planner (\Cref{sec:software}).

\subsection{Paper structure}
The remainder of this paper is organized as follows. We state the nonconvex minimum-time planning problem in~\Cref{sec:problem_statment}.~\Cref{sec:convexity_background} reviews the convex-analysis tools we build our approach on. We then set up our BMTP in three steps. In~\Cref{sec:minimum_time} we convexify the minimum-time objective and the derivative constraints. In~\Cref{sec:polar_reformulation} we recast collision avoidance through time-varying separating planes which yields a convex problem up to bilinearities in the separation constraints. In~\Cref{sec:sequential_convex_programming} we describe our BMTP for optimizing solutions to the bilinear formulation. In~\Cref{sec:discretization_convex} we describe how to make the problems finite-dimensional and solve them numerically using \bez curves.~\Cref{sec:experimental_eval} evaluates the method on drone navigation and dual-arm bin unloading in simulation and on hardware.~\Cref{sec:discussion} discusses our findings and limitations.~\Cref{sec:software} describes our open-source implementation, and we draw a conclusion in~\Cref{sec:conclusion}.

%% file: sections/problem_statement.tex
\section{Problem Statement}\label{sec:problem_statment}

We consider the problem of designing a trajectory that connects two points in minimum time, while avoiding a set of obstacles and satisfying given derivative constraints.
We represent the trajectory as the function
$$
q:[0,T] \rightarrow \mathbb R^n,
$$
where $T \in \mathbb R_{>0}$ is the trajectory duration and $n \in \mathbb N$ is the space dimension.
The fixed initial point is denoted as $q_0 \in \mathbb R^n$ and the final one as $q_T \in \mathbb R^n$.
The obstacles to be avoided are represented by the sets $\calO_1, \ldots, \calO_K \subset \mathbb R^n$.
We assume that these sets are convex and open, so that it is feasible to move along their boundary. Observe that nonconvex obstacles can often be exactly decomposed or efficiently approximated as unions of convex obstacles.

We ask that the trajectory $q$ be continuously differentiable $I \in \mathbb N$ times.
For $i=1,\ldots, I$, we let
$$
q^{(i)} :[0,T] \rightarrow \mathbb R^n
$$
be the $i$th time derivative of the trajectory $q$.
Each trajectory derivative $q^{(i)}(t)$ must be contained in a corresponding convex set $\mathcal C_i$ at all times $t \in [0, T]$.
These constraint sets are compact, convex, and contain the origin in their interior.
The initial and final values of the trajectory derivatives must be zero.

The motion-planning problem just described can be formulated as the following optimization problem:
\begin{subequations}
\label{eq:statement}
\begin{align}
\text{minimize} \quad & T \\
\text{subject to} \quad
& T > 0, \\
\label{eq:statement_boundary_configuration}
& q(0) = q_0, \ q(T) = q_T, \\
\label{eq:statement_boundary_derivative}
& q^{(i)}(0) = q^{(i)}(T) = 0, \quad i=1,\ldots, I, \\
\label{eq:statement_derivative}
& q^{(i)}(t) \in \mathcal C_i, \quad i=1,\ldots, I, \ t \in [0,T], \\
\label{eq:statement_obstacle_avoidance}
& q(t) \notin \calO_k, \quad k=1,\ldots, K, \ t \in [0,T].
\end{align}
\end{subequations}
The variables are the trajectory $q$ (which is infinite dimensional) and the time duration $T$.
The existence and continuity of the trajectory derivatives are implicit constraints in this problem.

%% file: sections/convexity_background.tex
\section{Background on Employed Convex Sets}\label{sec:convexity_background}

In this section, we introduce two families of convex sets that will play a central role in the remainder of the paper.
We use the notation $\lambda \mathcal S = \{\lambda x : x \in \mathcal S\}$ to denote the product of a scalar $\lambda \in \mathbb R$ and a set $\mathcal S \subseteq \mathbb R^n$.

\begin{lemma}
\label{lem:concave_perspective}
Let $\mathcal C \subseteq \mathbb R^n$ be a convex set that contains the origin.
Let $f : \mathbb R \rightarrow \mathbb R$ be a concave function.
The following set is convex:
$$
\mathcal S = \{(x, y): f(y) > 0, \ x \in f(y) \mathcal C\}.
$$
\end{lemma}

\begin{proof}
The set $\mathcal T = \{(x, y, z): f(y) \geq z > 0, x \in z \mathcal C\}$ can be shown to be convex as in~\cite[\S2.3.3]{boyd2004convex}.
The set $\mathcal S$ is convex since it is the orthogonal projection of $\mathcal T$ onto the space of the variables $x$ and $y$.
To see this, let $(x, y, z)$ be any point in $\mathcal T$.
Let $\bar x$ be the point in $\mathcal C$ such that $x = z \bar x$ and define $\theta = z / f(y) \in (0,1]$.
We have $x = z \bar x = f(y) \theta \bar x = f(y) (\theta \bar x + (1- \theta) 0)$.
Since both $\bar x$ and $0$ are in $\mathcal C$, so is their convex combination.
Hence $x \in f(y) \mathcal C$.
Together with $f(y) \geq z > 0$, this shows that $(x, y) \in \mathcal S$. Conversely, any $(x,y)\in\calS$ satisfies $(x,y,f(y))\in\calT$, so $\calS$ is the projection of $\calT$.
\end{proof}

To formulate collision avoidance constraints for an obstacle $\calO\subseteq\mathbb{R}^n$, it is useful to characterize the set of valid linear inequalities. We say the linear inequality $a^\top x + b\geq 0$ is \emph{valid} for $\calO$ if it holds for all points $x\in\calO$.
\begin{definition}\label{def:polarcone}
    The polar $\calO^\circ$ of a set $\calO\subseteq\mathbb{R}^n$ is the set of all vectors defining valid inequalities:
    \begin{gather}
        \calO^\circ:=\left\{(a,b):~a^\top x + b\geq 0 ~~\text{for all }x\in\calO\right\}.
    \end{gather}
\end{definition}
The polar $\calO^\circ$ is the intersection of one halfspace per point in $\calO$, and is therefore convex, even when $\calO$ is not. We have included a computational recipe for computing the polars of convex sets in conic standard form in~\Cref{app:conic_polar}.

%% file: sections/minimum_time.tex
\section{Joint Convexification of the Minimum-Time Objective and Derivative Constraints}\label{sec:minimum_time}

In this section, we formulate the minimum-time problem~\eqref{eq:statement} as an optimization problem whose nonconvexity is only due to the obstacle avoidance constraints~\eqref{eq:statement_obstacle_avoidance}.
The technique we use is closely related to the one from~\cite{leomanni2022time}.

We work with a normalized version $r:[0,1] \rightarrow \mathbb R^n$ of the function $q$ that has domain $[0,1]$ instead of domain $[0,T]$.

The trajectory $q$ is obtained from $r$ simply as
$$
q(t) = r(t/T)
$$
for all $t \in [0, T]$.
Similarly, the time derivatives of $q$ are computed as
$$
q^{(i)}(t) = \frac{1}{T^i} r^{(i)}(t/T)
$$
for all $t \in [0, T]$ and $i=1, \ldots, I$.

With this new trajectory parameterization, the boundary conditions~\eqref{eq:statement_boundary_configuration} and~\eqref{eq:statement_boundary_derivative} are almost unchanged:
\begin{subequations}
\begin{align}
& r(0) = q_0, \ r(1) = q_T, \\
& r^{(i)}(0) = r^{(i)}(1) = 0, \quad i=1,\ldots, I.
\end{align}
\end{subequations}
Note that in the second family of constraints we eliminated the terms $T^i$ since the right-hand side is zero.

The derivative constraints~\eqref{eq:statement_derivative} become
\begin{align}
& r^{(i)}(s) \in T^i \mathcal C_i, \quad i=1,\ldots, I, \ s \in [0,1],
\end{align}
where we multiplied both sides by $T^i$.
This family of constraints is nonconvex, but it can be convexified by working with the variable $T_I = T^I$ instead of $T$.
Then the constraints above become
\begin{align}
& r^{(i)}(s) \in T_I^{i/I} \mathcal C_i, \quad i=1,\ldots, I, \ s \in [0,1],
\end{align}
which are convex by~\Cref{lem:concave_perspective}.
Note that the function $T_I^{i/I}$ is concave and the set $\mathcal C_i$ contains the origin, for $i=1,\ldots, I$.

In terms of the variable $T_I$, the objective of our problem reads $T_I^{1/I}$.
This is not a convex function of $T_I$, but it is monotonically increasing in $T_I$.
Therefore, we can simply minimize $T_I$, without affecting the optimal solution of our problem.

Overall, we have the following optimization problem
\begin{subequations}\label{eqn:convex_min_time}
\begin{align}
\text{minimize} \quad & T_I \label{eqn:cvx_mintime_cost}\\
\text{subject to} \quad
& T_I > 0, \label{eqn:cvx_mintime_cons_start}\\
& r(0) = q_0, \ r(1) = q_T, \label{eqn:cvx_mintime_endpoint_cons}\\
& r^{(i)}(0) = r^{(i)}(1) = 0, \quad i=1,\ldots, I, \label{eqn:cvx_mintime_endpoint_cons2}\\
& r^{(i)}(s) \in T_I^{i/I} \mathcal C_i, \quad i=1,\ldots, I, \ s \in [0,1], \label{eqn:cvx_mintime_cons_end}\\
\label{eqn:col_avoidance}
& r(s) \notin \calO_k, \quad k=1,\ldots, K, \ s \in [0,1].
\end{align}
\end{subequations}
The variables are the function $r$ and the scalar $T_I$.
The objective function is linear.
All the constraints are convex, except for the collision avoidance.

%% file: sections/collision_avoidance.tex
\section{Polar Reformulation of Collision-Avoidance Constraints}\label{sec:polar_reformulation}
In this section, we perform one last reformulation before introducing our BMTP for finding solutions to the minimum-time problem~\eqref{eqn:convex_min_time}.

We reformulate the collision avoidance constraints~\eqref{eqn:col_avoidance} as a search for separating planes between the trajectory and each obstacle $\calO_k$. More precisely, for each obstacle $\calO_k$, we search for valid planes $(a_k, b_k): [0,1]\rightarrow \calO_k^\circ$ that separate the path $r$ from each obstacle $\calO_k$ for all $k=1,\ldots, K$ and $s \in [0,1]$. This turns problem~\eqref{eqn:convex_min_time} into the equivalent optimization problem:
\begin{subequations}\label{eqn:biconvex_problem}
\begin{align}
\text{minimize}\quad&\eqref{eqn:cvx_mintime_cost}\label{eqn:biconvex_problem_cost}\\
\text{subject to}\quad&\eqref{eqn:cvx_mintime_cons_start}-\eqref{eqn:cvx_mintime_cons_end},\\
& a_k(s)^\top r(s) + b_k(s) < 0 , \ s \in [0,1],\label{eqn:biconvex_problem_avoidance}\\
& (a_k(s), b_k(s)) \in \calO_k^\circ, \ s \in [0,1],\label{eqn:biconvex_polar}
\end{align}
\end{subequations}
where we add one set of plane constraints~\eqref{eqn:biconvex_problem_avoidance}, \eqref{eqn:biconvex_polar} for each $k=1,\ldots,K$.
We use a strict inequality in~\eqref{eqn:biconvex_problem_avoidance} to prevent selecting $(a_k(s), b_k(s))=0$ which would not enforce separation. The constraint~\eqref{eqn:biconvex_polar} is convex, therefore the only remaining nonconvexities are the bilinearities in $a_k$ and $r$. 

We now observe that the problem is convex if we either fix the path $r$ or the planes $(a_k, b_k)$. To find good solutions to~\eqref{eqn:convex_min_time}, we can now apply standard techniques like alternate convex search or block coordinate descent as described in~\cite{gorski2007biconvex, shen2017disciplined} and references therein. In our case, the analogous idea is alternating between searching for valid planes given a fixed trajectory, and then subsequently refining the trajectory while fixing the planes.

The idea of optimizing separating planes jointly with the trajectory also appears in~\cite{zhang2020optimization, tordesillas2021mader}. These works enforce separation only at discrete time steps~\cite{zhang2020optimization} or with a single constant plane per trajectory segment~\cite{tordesillas2021mader}, and solve the joint problem with general nonconvex solvers. In contrast, our planes are time-varying functions that separate the trajectory from each obstacle at every $s \in [0,1]$.

%% file: sections/sequential_convex_programming.tex
\section{Biconvex Minimum-Time Planner}\label{sec:sequential_convex_programming}
\begin{figure}
    \centering
    \includegraphics[width=\linewidth, trim= {0cm, 0.cm, 0cm, 0cm}, clip]{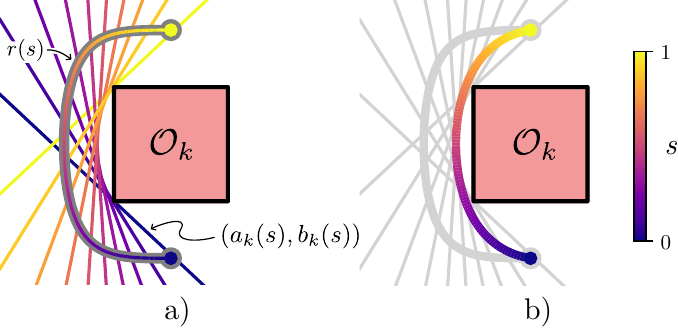}
    \caption{A single alternation between the plane update (left) and the trajectory update (right). During the plane update, we compute separating planes that separate each point on the current trajectory from the obstacle while maximizing the margin, leaving room for the trajectory to improve. During the trajectory update, we re-optimize the trajectory while keeping the planes fixed.}
    \label{fig:single_alternation}
\end{figure}

\begin{figure*}[t]
    \centering
    \includegraphics[width=\linewidth]{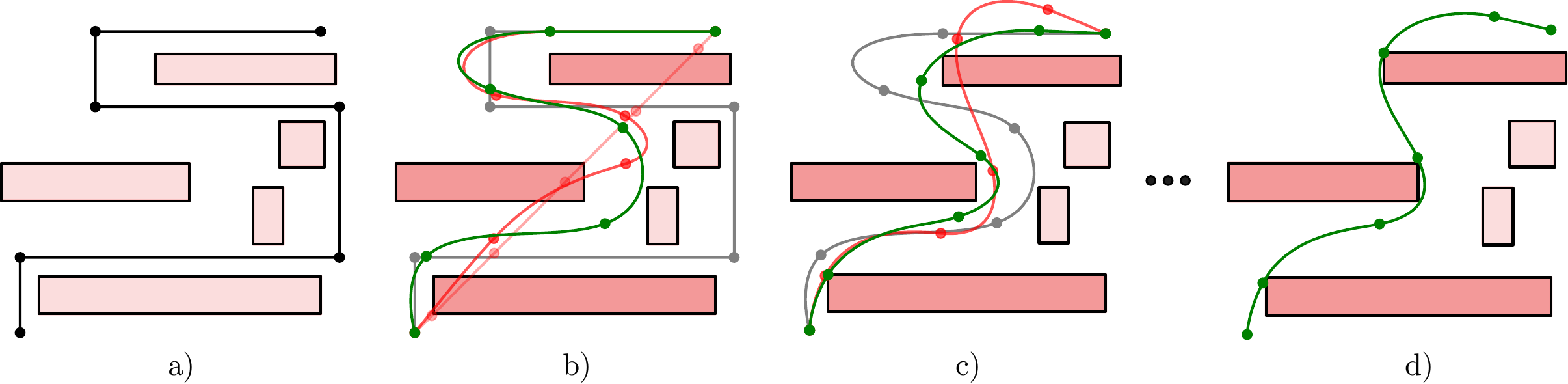}
    \caption{Our BMTP applied to a 2D environment with four obstacles. The last feasible trajectory is shown in gray, infeasible intermediate iterates are shown in red, and the next feasible iterate is shown in green. Tagged obstacles are shaded more strongly. In frame a), we initialize with a collision-free trajectory $r_\text{init}$. In frame b), the first iteration proceeds through several infeasible iterates: the initial red trajectory cuts directly toward the goal while intersecting three obstacles, followed by a curved red trajectory that collides with the middle obstacle on a different segment, and finally the first feasible iterate that jumps around the two smaller obstacles on the right. In frame c), we update both the feasible trajectory and its separating planes, giving the trajectory more room to improve. One additional collision occurs before finding the next feasible iterate. In frame d), the algorithm has converged to the final optimized trajectory.}
    \label{fig:sco_overview}
\end{figure*}
We are now ready to introduce our biconvex minimum-time planner (BMTP), outlined in \Cref{alg:sco_approach}.

First, we observe that using the polar reformulation of the collision avoidance constraints, problem~\eqref{eqn:biconvex_problem} is now biconvex: If we either fix the separating planes $(a_k,b_k)$ or the path $r$, the optimization becomes convex. We can therefore employ a basic biconvex procedure as in~\cite{gorski2007biconvex, shen2017disciplined} for solving the trajectory optimization. First, we initialize our search with a feasible path $r$. We then alternate between fixing the path and updating the separating planes, which we call the \emph{plane update} and cover in~\Cref{ssec:plane_update}, and fixing the separating planes and updating the path and its traversal time, which we call the \emph{trajectory update} and cover in~\Cref{ssec:trajectory_update}. A single alternation is illustrated in~\Cref{fig:single_alternation}. 

Our full biconvex minimum-time planner adds a mechanism akin to~\cite[\S3.3]{lipp2016variations} on top of these basic alternations that allows the trajectory to escape some local minima and jump around obstacles. The full procedure is outlined in~\Cref{alg:sco_approach} and detailed in~\Cref{ssec:biconvex_highlevel_prod}. The idea is that there may be many obstacles, but only a few of the collision-avoidance constraints will be active at the converged solution. Therefore, it can be beneficial to only add and update planes for obstacles that the trajectory collides with during the alternations. In the remainder of the paper we will refer to these obstacles as ``tagged obstacles''.

\subsection{High-level procedure}\label{ssec:biconvex_highlevel_prod}
We initialize the procedure with a feasible, collision-free normalized trajectory $r = r_\text{init}$, which we subdivide into $M$ equal duration pieces, an empty collection of separating planes $\calP$, as well as an empty index set $\calI$ that will track which trajectory piece has collided with which obstacle. Then we proceed to an inner loop where we repeatedly solve the trajectory update, producing a candidate trajectory $c:[0,1] \rightarrow \mathbb R^n$ for the current planes and check $c$ for collisions. If the $m$th segment of $c$ collides with an obstacle $k$, we compute the plane update for the newly tagged obstacle $k$ around the corresponding $m$th segment of the last feasible trajectory $r$. If no collisions are found, we replace the last feasible trajectory with the new candidate $c$ and update all activated planes around the new feasible trajectory. We found that dividing $r$ into $M$ pieces and enforcing the collision avoidance constraint on an entire segment if a collision is found provides an effective middle ground between enforcing the constraint only at the $s$ where a collision was found and enforcing it on the entire trajectory.

Such an initialization can be obtained from a collision-free polygonal path by solving~\eqref{eqn:convex_min_time} between each pair of consecutive waypoints, dropping the collision-avoidance constraints~\eqref{eqn:col_avoidance} and constraining the path to the line segment connecting the waypoints. Since all trajectory derivatives vanish at the segment endpoints, the timed segments concatenate into a feasible trajectory that is continuously differentiable $I$ times.

\Cref{fig:sco_overview} provides an overview of this procedure. In frame a) we see the initial collision-free trajectory $r_\text{init}$ with $M=6$ segments. We then see the first iteration in frame b). The initial red trajectory cuts directly from the start to the goal while intersecting the three highlighted obstacles, which is followed by the curved red trajectory that is now intersecting the middle obstacle with a different segment, and finally the first feasible iterate that manages to jump around the two smaller obstacles to the right. In frame c) we see the second iteration. We update both the feasible trajectory, now shown in gray, and its planes which gives the trajectory more room to improve. During this iteration there is one additional collision before finding an improvement. Finally, in frame d), we see the converged trajectory.

\begin{algorithm}
\SetAlgoLined
\caption{Biconvex Minimum-Time Planner}
\label{alg:sco_approach}
\SetKwInput{Input}{Input}
\SetKwInput{Output}{Output}
\SetKw{KWAlgorithm}{Algorithm:} 
\Input{
Feasible path $r_\text{init}$
}
\Output{
Optimized, collision-free trajectory $(r, T)$.
}

\KWAlgorithm{}\\
$\calP \gets \emptyset$, $\calI \gets \emptyset$, $r\gets r_\text{init}$\\
\While{not converged
}{
    \Repeat{ $\calI_c=\emptyset$}{
   $(c, T_c)\gets \textsc{TrajectoryUpdate}(\calP)$\\
   $\calI_c \gets\textsc{GetObstacleCollisions}(c)$\\
   $\calI\gets \calI\cup\calI_c$\\
   $\calP \gets \calP\cup\textsc{PlaneUpdate}(\calI_c,r)$\\
   }
   $(r,T)\gets (c, T_c)$\\

   $\calP\gets\textsc{PlaneUpdate}(\calI,r)$
}
\Return $(r,T)$
\end{algorithm}

\subsection{Plane update}\label{ssec:plane_update}
We first note that if we fix the path $r$ in~\eqref{eqn:biconvex_problem}, any set of feasible planes produces the same overall trajectory cost. Unfortunately, this program can, therefore, produce planes that lie far from obstacles and very close to the trajectory. While such planes are a perfectly valid solution to the program, they are not desirable because they leave the trajectory with little to no room for improvement in subsequent alternations. Instead, we use a modified heuristic cost for~\eqref{eqn:biconvex_problem} during the plane update that pushes the planes as far away from the trajectory as possible while ensuring that they still separate the trajectory from the obstacles. 

More precisely, given a fixed, collision-free trajectory $r$, the $k$th obstacle $\calO_k$, and set of all segment indices $\calI_k$ of candidate trajectories $c$, for which the $k$th segment collided with $\calO_k$, let 
\begin{align}
    \mathcal{S}_k := \bigcup_{m \in \mathcal{I}_k} \left[\frac{m}{M}, \frac{m+1}{M}\right].
\end{align}

The plane functions $(a^\star_k, b^\star_k): \calS_k\rightarrow \calO^\circ_k$ that maximize the margin to the trajectory $r$ are given by projecting $r$ onto the obstacle $\calO_k$ and choosing the supporting planes with normals that point from $r$ to its projection for all $s$ in $\calS_k$:
\begin{subequations}\label{eqn:continuous_max_margin_planes_solution}
\begin{align}
a^\star_k(s) &= \frac{\mathrm{proj}_{\calO_k}[r(s)] - r(s)}{\|\mathrm{proj}_{\calO_k}[r(s)] - r(s)\|_2}, \\
b^\star_k(s) &= -a^{\star}_k(s)^\top \mathrm{proj}_{\calO_k}[r(s)].
\end{align}
\end{subequations}



\subsection{Trajectory update}\label{ssec:trajectory_update}
For the trajectory update, we can simply re-optimize the normalized trajectory $r$, and its timing $T$, by fixing the plane functions and solving~\eqref{eqn:biconvex_problem} which is now convex. This yields the candidate normalized trajectory $c(s)$ and its duration $T_c$. 

\subsection{Convergence}\label{ssec:convergence}
In this section we prove that our BMTP is guaranteed to converge under the assumption that the plane update is feasible for all segment obstacle pairs that are not colliding. This assumption holds for any collision-free polygonal initial trajectory, because each segment-obstacle pair can be separated by a plane that is constant in time.

First, we observe that the objective $T$ never increases across outer iterations. This holds because the trajectory update is a convex optimization, and the previous feasible trajectory $r$ and its timing always remain in the feasible set of the subproblem. The produced feasible candidates $c$ can therefore never have higher cost. Additionally, the trajectory duration can be lower-bounded by $T=0$. Therefore, the sequence of trajectory durations of the feasible candidates must converge.  

To conclude our proof, we confirm that the inner iterations always terminate in a finite number of iterations. We can directly see that the inner loop is guaranteed to find a new feasible trajectory after at most $KM$ attempts, because once a pair $(k,m)$ is tagged, the plane update ensures that the pair can never collide in later iterations. Each inner iteration therefore tags at least one new pair from the pool of $KM$ obstacle-segment combinations or terminates with a new feasible candidate $c$.

These properties also make BMTP anytime. Every feasible iterate $(r, T)$ satisfies all constraints of the minimum-time problem~\eqref{eq:statement}, and its duration never exceeds that of any earlier feasible iterate. If BMTP is interrupted before convergence, it returns the last feasible iterate.

%% file: sections/discretization.tex
\section{Finite-Dimensional Formulation}\label{sec:discretization_convex}
In this section, we make the subproblems of BMTP finite-dimensional by leveraging \bez curves, which allows us to implement the approach on a computer. \bez curves, and more generally B-splines, are widely used in motion planning, from early work~\cite{flores2008real} to more recent approaches such as listed in~\Cref{ssec:dbmp_overview}.

We represent the normalized trajectory $r$ through a composite \bez curve with $M$ segments, each of duration $1/M$. All segments in the final trajectory therefore have the same duration $T/M$. This is in contrast to the DBMP formulations, where each segment carries its own duration, and allows us to formulate the continuity constraints between segments and their derivatives as linear equality constraints. We also represent the time-varying planes as multi-dimensional composite \bez curves. In particular, we split the domain of the functions $(a_k,b_k)$ into $M$ segments $(a_{m,k},b_{m,k})$, each of duration $1/M$. Since we only need to compute the planes for segments of the trajectory that have tagged obstacles, we assign an individual \bez curve to represent the corresponding segment $m$ of the plane function, while simply not adding constraints~\eqref{eqn:biconvex_problem_avoidance},~\eqref{eqn:biconvex_polar} for the other segments. We note that the plane functions need not be continuous across segments, so each segment of the plane functions can be optimized independently and in parallel.

\bez curves have a number of desirable properties which we leverage in our implementation. In the remainder of this section, we first briefly summarize some useful properties of \bez curves, then we apply these properties to the trajectory update and the plane update giving us the complete finite-dimensional formulation of the approach.

\subsection{Useful properties of \bez curves}
A degree-$D$ \bez curve is a vector-valued curve in $\mathbb R ^n$ given by
$$B(s) = \sum_{d = 0} ^D \beta_d(s)\pi_d, \quad \text{for } s\in[0,1],$$
where $\pi_d\in\mathbb R ^n$ are control points and $\beta_d$ are the degree-$D$ Bernstein polynomials

$$ \beta_d(s) = \begin{pmatrix}
    D\\d
\end{pmatrix}(1-s)^{D-d}s^d, \quad \text{for } d= 0,\dots, D.$$
For a fixed degree $D$, the Bernstein polynomials sum to one due to the binomial theorem, and are nonnegative for $s\in[0,1]$. Therefore, \bez curves can be thought of as forming a convex combination of the control points $\pi_d$ weighted by $\beta_d$. We now list some useful properties of \bez curves. For a more comprehensive overview see~\cite{farouki1988algorithms}.

\begin{property}[Convex Hull]\label{prop:cvx_hull}
    A degree-$D$ \bez curve $B(s)$ is contained in the convex hull of its control points $B(s)\subseteq\conv\{\pi_0,\dots, \pi_D\}, ~~\text{for all}~ s\in[0,1].$
\end{property}
\begin{property}[Derivative]\label{prop:derivative}
The derivative of a degree-$D$ \bez curve is a  \bez curve of degree $D-1$ with control points $D(\pi_{d+1} - \pi_d)$ for $d = 0, \dots, D-1$. 
\end{property}

\begin{property}[End Point]\label{prop:endpoint}
    A \bez curve of degree $D$ interpolates the first and the last control point: $B(0) = \pi_0$ and $B(1) = \pi_D$.
\end{property}
\begin{property}[Summation]\label{prop:summation}
    Given two \bez curves $B_1$ and $B_2$, both of degree $D$, their summation $B_{\text{sum}}(s) = B_1(s) + B_2(s)$ is a degree-$D$ \bez curve whose control points are given by the sum $\pi_{\text{sum},d} = \pi_{1,d} + \pi_{2,d}$. If the curves have different degrees, they can be summed by elevating the lower-degree curve to the same degree using the linear transformation on its control points given in~\cite[\S3.2]{farouki1988algorithms}.
\end{property}
\begin{property}[Multiplication]\label{prop:mult}
    Given scalar \bez curves $B_1$, of degree $D_1$, and $B_2$ of degree $D_2$, their multiplication $B_{\text{prod}}(s) = B_1(s)B_2(s)$ is a degree $D_1+D_2$ \bez curve. Its control points $\pi_{\text{prod}, d}$ are given by 
    $$ 
    \pi_{\text{prod}, d} = \sum_{l = \max(0, d-D_1)}^{\min(D_2,d)}\frac{\begin{pmatrix}
        D_2\\l
    \end{pmatrix}\begin{pmatrix}
        D_1\\d-l
    \end{pmatrix}}{\begin{pmatrix}
        D_1+D_2\\d
    \end{pmatrix}}\pi_{1,l}\pi_{2,d-l},
    $$
for $d = 0, \ldots, D_1+D_2$, as discussed in~\cite[\S4.2]{farouki1988algorithms}.
\end{property}

\subsection{Finite-dimensional trajectory update}\label{ssec:fd_trajectory_update}

We represent the normalized trajectory $r$ as a \bez curve. We enforce the constraints on the endpoints of the trajectory~\eqref{eqn:cvx_mintime_endpoint_cons}, and the endpoints of its first $I$ derivatives~\eqref{eqn:cvx_mintime_endpoint_cons2} using~\Cref{prop:derivative,prop:endpoint}. By~\cref{prop:cvx_hull}, we can enforce the derivative constraints~\eqref{eqn:cvx_mintime_cons_end} on the entire spline by only enforcing them on the control points of each segment of the derivative spline (\Cref{prop:derivative}). The implicit continuity constraints on the first $I$ derivatives are enforced via equality constraints on the control points of the derivative spline (\Cref{prop:derivative,prop:endpoint}).

There are multiple options for enforcing the collision-avoidance constraints~\eqref{eqn:biconvex_problem_avoidance}, given the precomputed planes $\calP$. We begin by evaluating each segment of the computed plane functions 
\begin{align}
    v_{m,k}(s) =  a_{m,k}(s)^\top r_m(s) + b_{m,k}(s).
\end{align}
The function $v_{m,k}$ is a polynomial that has a degree equal to the sum of the degrees of the trajectory and the plane curves. Here, $r_m$ denotes the $m$th segment of the spline representing the trajectory. The collision avoidance constraints now take the form
\begin{align}\label{eqn:scalar_curve_ineq}
    v_{m,k}(s)<0, \quad\text{for }s\in\left[\frac{m}{M},\frac{m+1}{M}\right].
\end{align}

We use a cheap, conservative approximation to enforce~\eqref{eqn:scalar_curve_ineq}. We use~\Cref{prop:summation,prop:mult} to convert $v_{m,k}$ to a scalar \bez curve in Bernstein form. This corresponds to taking a linear transformation of the trajectory decision variables. By~\cref{prop:cvx_hull}, we can now enforce a conservative approximation of~\eqref{eqn:biconvex_problem_avoidance} by constraining all the control points of $v_{m,k}$ to be negative.

We note that this condition can be enforced exactly in a convex fashion by using sums-of-squares optimization \cite[\S 3]{parrilo2003semidefinite} via the Markov–Lukács theorem~\cite{roh2006discrete}. This technique has been used in the past to certify the safety of trajectories~\cite{amice2024certifying}, but we found it to be too computationally expensive.

\subsection{Finite-dimensional plane update}
It is not practical to directly construct the separation certificates via the closed-form solution given in~\eqref{eqn:continuous_max_margin_planes_solution}. The projection of $r$ onto the obstacle may not result in a polynomial trajectory even when $r$ is, and simple heuristics such as only projecting the control points of $r$ may result in planes that are not contained in the polar set.

Instead, we use a more pragmatic approach: For each tagged obstacle and each segment that has collided with it, we split the plane update $\calO_k$ into independent subproblems.
If the $m$th segment has collided with the $k$th obstacle, we approximately solve the surrogate problem

\begin{subequations}
\label{eqn:max_margin_planes_discrete}
\begin{align}
\text{minimize} \quad & \max_{s\in\left[\frac{m}{M}, \frac{m+1}{M}\right]} v_{m,k}(s),\label{eqn:max_margin_planes_cost_discrete}\\
\text{subject to} \quad
&  (a_{m,k}(s), b_{m,k}(s)) \in \calO^\circ_k, \label{eqn:max_margin_planes_polar_cons_discrete}\\
& \| a_{m,k} \|_2 \leq 1,\label{eqn:max_margin_planes_norm_cons_discrete}
\end{align}
\end{subequations}
by conservatively minimizing the largest control point of $v_{m,k}$, which upper-bounds the maximum by~\cref{prop:cvx_hull}, and enforcing the constraints~\eqref{eqn:max_margin_planes_polar_cons_discrete},~\eqref{eqn:max_margin_planes_norm_cons_discrete} on the control points of the corresponding \bez curves.~\cref{prop:cvx_hull} ensures that these constraints hold on the plane function for all $s\in\left[\frac{m}{M}, \frac{m+1}{M}\right]$. 
We found this heuristic empirically crucial. It tries to maximize the margin between the planes and the trajectory where it is tightest. This ensures, when possible, that the subsequent trajectory update has some margin to improve. We further observe that as the number of segments tends towards infinity, this corresponds to a point-wise minimization of the margin, just as the infinite-dimensional formulation in~\eqref{eqn:continuous_max_margin_planes_solution}. See~\Cref{app:max_margin_plane_dual} for its connection to this discrete program.

\subsection{Practical details}\label{ssec:practical_details}
In practice, the obstacles are often closed sets. To ensure strict separation, we add a small positive margin $\Delta_v>0$ to~\eqref{eqn:scalar_curve_ineq}:
\begin{align}
    v_{m,k}(s)\leq-\Delta_v, \quad\text{for }s\in\left[\frac{m}{M},\frac{m+1}{M}\right].
\end{align}  
Similarly, to ensure that the planes are strictly separating when they have nonzero magnitude, we slightly modify the polar constraints to add a small positive margin $\Delta_p$ when the planes have nonzero magnitude. As explained in~\Cref{app:conic_polar}, for a convex set $\calO$ in conic standard form,
$$\calO = \{x\inR^n \mid \exists y\inR^m : Ex + Fy + g \in \calK\},$$ 
where $\calK$ is a closed convex cone, this can be accomplished by enforcing the polar constraints on the decision variables $(a,b)$ by adding the constraints
\begin{equation}
\calO^\circ = \left\{(a, b)\inR^{n+1} ~\middle|~
\begin{array}{l}
a = E^\top\nu,~ F^\top\nu = 0,\\
b \geq g^\top\nu+\Delta_p,~ \nu \in \calK^*
\end{array}
\right\},
\end{equation}
where $\calK^*$ denotes the dual cone of $\calK$. In practice, we find that values of $\Delta_p = \Delta_v = 10^{-6}$ tend to work well.

For further practical details, we refer the reader to our open-source implementation (\Cref{sec:software}).

%% file: sections/experiments.tex
\section{Experimental Evaluation}\label{sec:experimental_eval}

In this section we evaluate our approach on a series of experiments. First, we present an illustrative experiment where we plan a minimum-time trajectory for a drone to efficiently traverse a village in \Cref{ssec:village}. This experiment demonstrates how our approach can escape bad initialization and handle higher-order derivative and smoothness constraints. We then compare our approach to a state of the art DBMP on a bimanual bin unloading task in simulation in~\Cref{ssec:pallet_unloading} and validate the results on hardware in~\Cref{ssec:Hardware_validation}. For all experiments we employ an AMD Ryzen 9 7950X3D CPU.

\subsection{Drone flying through village}\label{ssec:village}

\begin{figure}[t]
    \centering
    \includegraphics[width=\linewidth, trim=11cm 1cm 5cm 1cm, clip, angle=0.1]{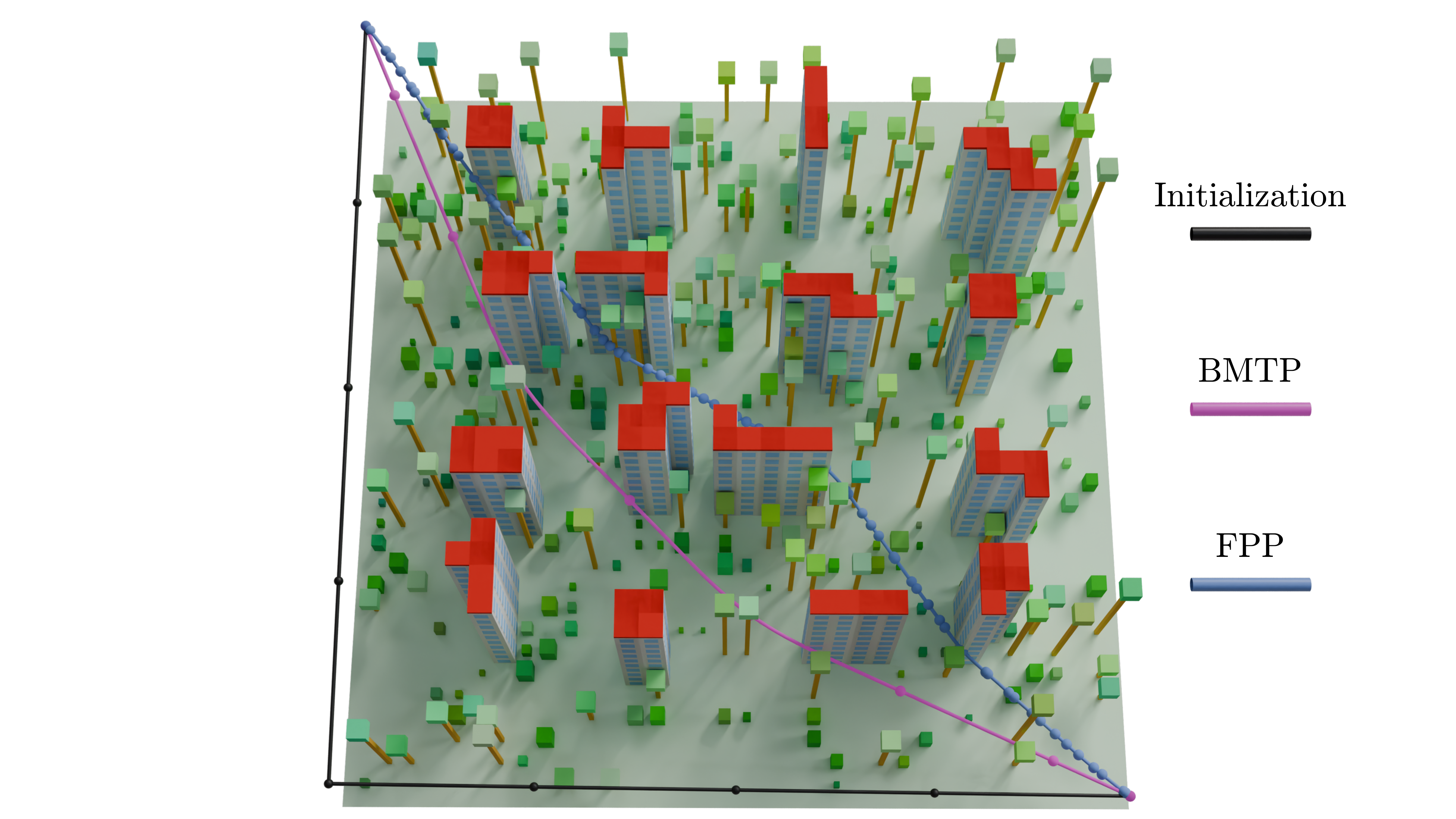}
    \caption{Quadrotor trajectory through the village environment from~\cite{marcucci2024fast}. The black line shows the initialization, which routes the drone around the village. The magenta curve shows the optimized BMTP trajectory, which cuts diagonally through the village despite the naive initialization.
    The blue trajectory shows the solution obtained by FPP~\cite{marcucci2024fast}, modified to support the minimum-time objective and using backtracking line search on the tangent step. The optimized trajectories are $C^4$-continuous with hard velocity, acceleration, jerk, and snap constraints.}
    \label{fig:village}
\end{figure}

In our first experiment, we plan a minimum-time trajectory for a quadrotor flying through the village environment from~\cite{marcucci2024fast} as shown in~\Cref{fig:village}. We require the quadrotor trajectory to be four-times continuously differentiable, as well as obey velocity, acceleration, jerk, and snap constraints, to guarantee dynamic feasibility through differential flatness~\cite{mellinger2011minimum}. This experiment highlights two key strengths of BMTP. First, it naturally handles high-order derivative and continuity constraints while simultaneously optimizing the trajectory duration. Second, it can jump around obstacles to escape poor initializations.

The village environment consists of 521 axis-aligned box obstacles (buildings, trees, and bushes) on a $34 \times 34$ meter grid. The start and goal are placed at opposite corners of the village, requiring the drone to traverse the environment diagonally. We initialize with a simple polygonal path with 8 segments that routes entirely around the village, avoiding all obstacles. We select trajectory degree 8, and enforce that the trajectory is four times continuously differentiable along with derivative constraints of the form
$$\calC_i = \left\{q^{(i)}(t) \mid \|q^{(i)}(t)\|_2\leq q^{(i)}_\text{max}\right\},$$
 for $i = 1,2,3,4 $, i.e., velocity, acceleration, jerk, and snap. We cap the velocity at $q^{(1)}_\text{max} = 5\,\mathrm{m/s}$, the acceleration at $q^{(2)}_\text{max} = 5\,\mathrm{m/s^2}$, jerk at $q^{(3)}_\text{max} = 25\,\mathrm{m/s^3}$, and snap at $q^{(4)}_\text{max} = 50\,\mathrm{m/s^4}$.

In order to run BMTP, we employ the collision-checking approach from~\Cref{app:collision_checking}, parallelized across all trajectory segments, and solve the resulting programs with Clarabel~\cite{goulart2026clarabel}. BMTP converges in 6 iterations of which 3 are feasible. The resulting collision-free trajectory is computed in $0.19$ seconds and has a duration of $11.83$ seconds, as reported in the first row of~\Cref{tab:village}.

For comparison, we also implement the minimum-time variation
of the Fast Path Planning (FPP) algorithm mentioned in~\cite[\S VIII]{marcucci2024fast}, but not available within the FPP software package. To improve convergence of this minimum-time variant of FPP, we modify the trust-region update rule as described in~\Cref{app:fpp}.
FPP matches the quality, but requires significantly higher computation time. In the first step of the algorithm, it selects a $70$-box corridor from the $2345$ safe boxes, and then solves the resulting $70$-segment smoothing problem. We observe this smoothing step to be sensitive to the trust-region schedule. The aggressive schedule from the original paper converges quickly, in $1.05$ seconds, but only to a high-cost $25.78$ second trajectory. We propose a small modification, which adds a backtracking line search that takes the largest projection-feasible step. With this modification the FPP approach results in an $11.68$ second trajectory in $10.28$ seconds. All FPP approaches require a $1.64$ second upfront cost paid for preprocessing all safe boxes. The times required in the polygonal phase of the algorithm are negligible in comparison to the preprocessing and the smooth phase. See~\Cref{app:fpp} for further details on our application of FPP.

\begin{table}[t]
\centering
\caption{Village Results}
\label{tab:village}
\resizebox{\columnwidth}{!}{%
\begin{tabular}{lccc}
\toprule
  & Segments & Time [s] & Duration [s] \\
\midrule
BMTP (ours)                 & 8  & $\mathbf{0.19}$ & 11.83 \\
FPP, original trust region  & 70 & 1.05  & 25.78 \\
FPP, modified trust region  & 70 & 10.28 & $\mathbf{11.68}$ \\
\bottomrule 
\end{tabular}}
\end{table}

The $8$-segment problem, where all the obstacles are ignored, can be solved as a convex program. For reference, this program gives us a lower bound on the BMTP trajectory duration of 11.35 seconds, showing that our method incurs only 0.48 additional seconds to avoid all the obstacles in the village.

Overall we find that our BMTP approach computes trajectories of equivalent quality to the method from~\cite{marcucci2024fast} while producing them over 50 times faster on this example. We do note that FPP scales well with the number of obstacles, but requires an upfront decomposition of the free space into boxes. 

\subsection{Bin unloading task with convex obstacles}\label{ssec:pallet_unloading}
\newcommand{\expfig}[1]{\includegraphics[width=0.24\textwidth, trim=9cm 3cm 9cm 0cm]{#1}}

\begin{figure*}[t]
    \centering
    \expfig{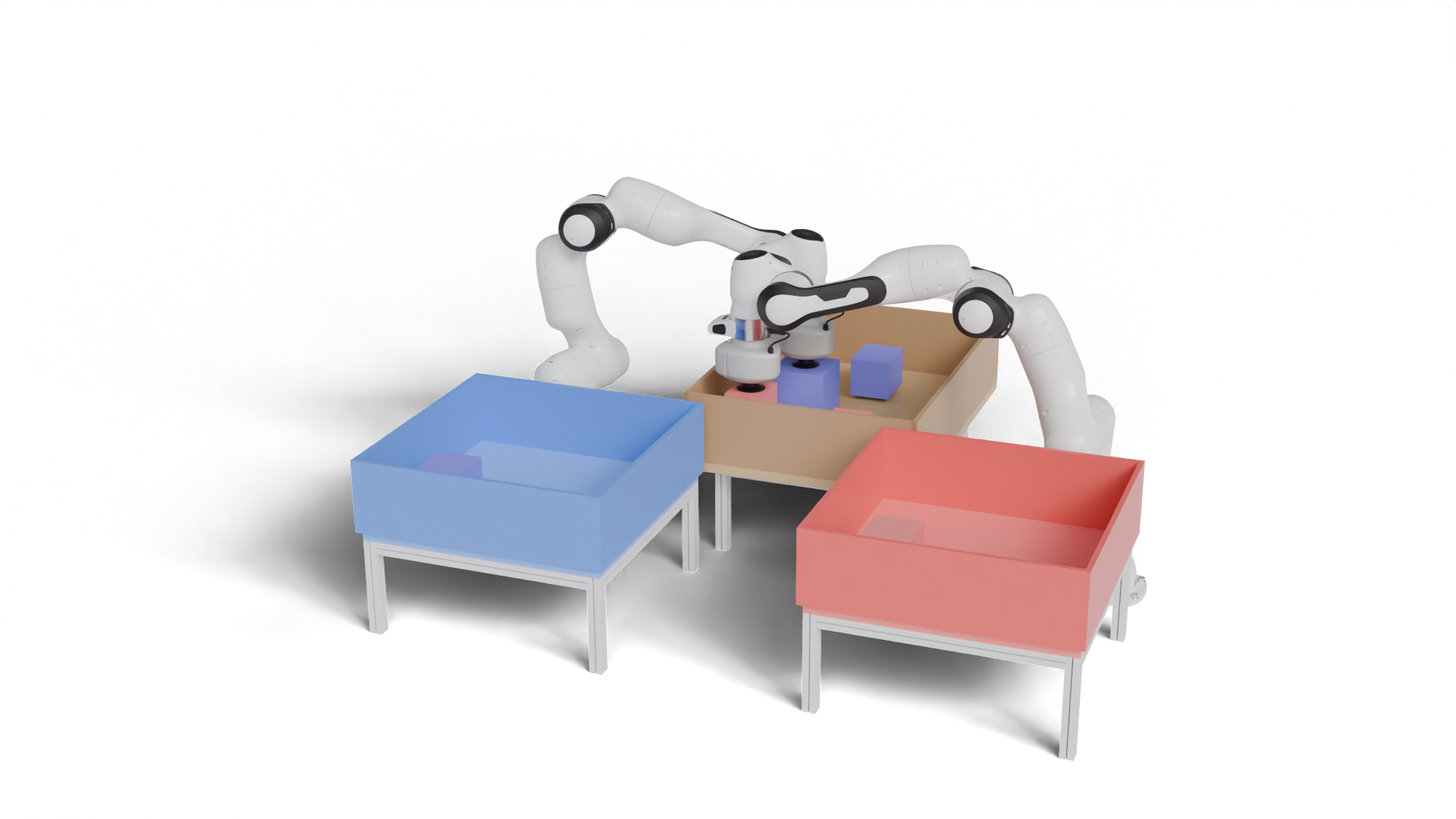}
    \hfill
    \expfig{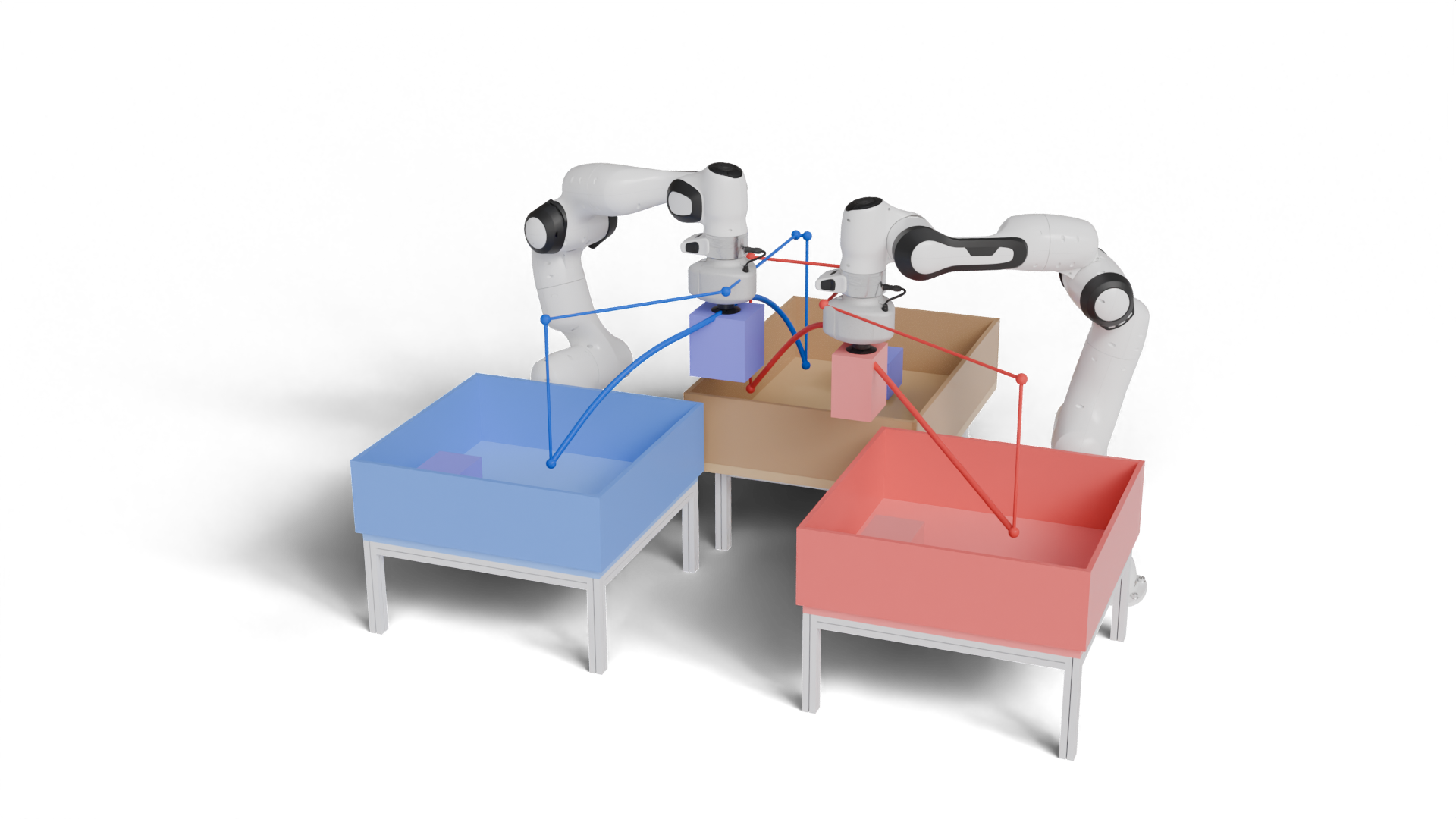}
    \hfill
    \expfig{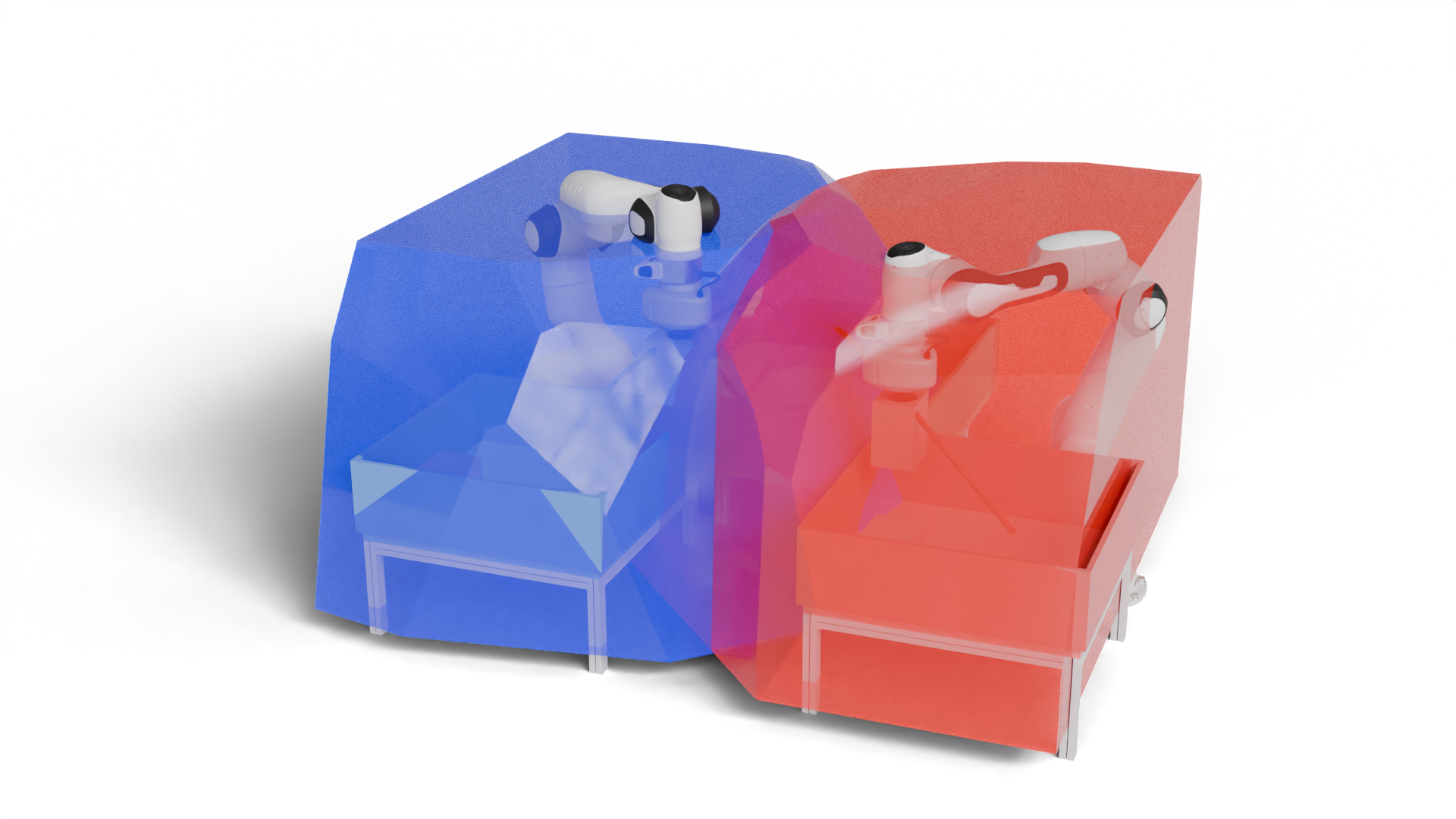}
    \hfill
    \expfig{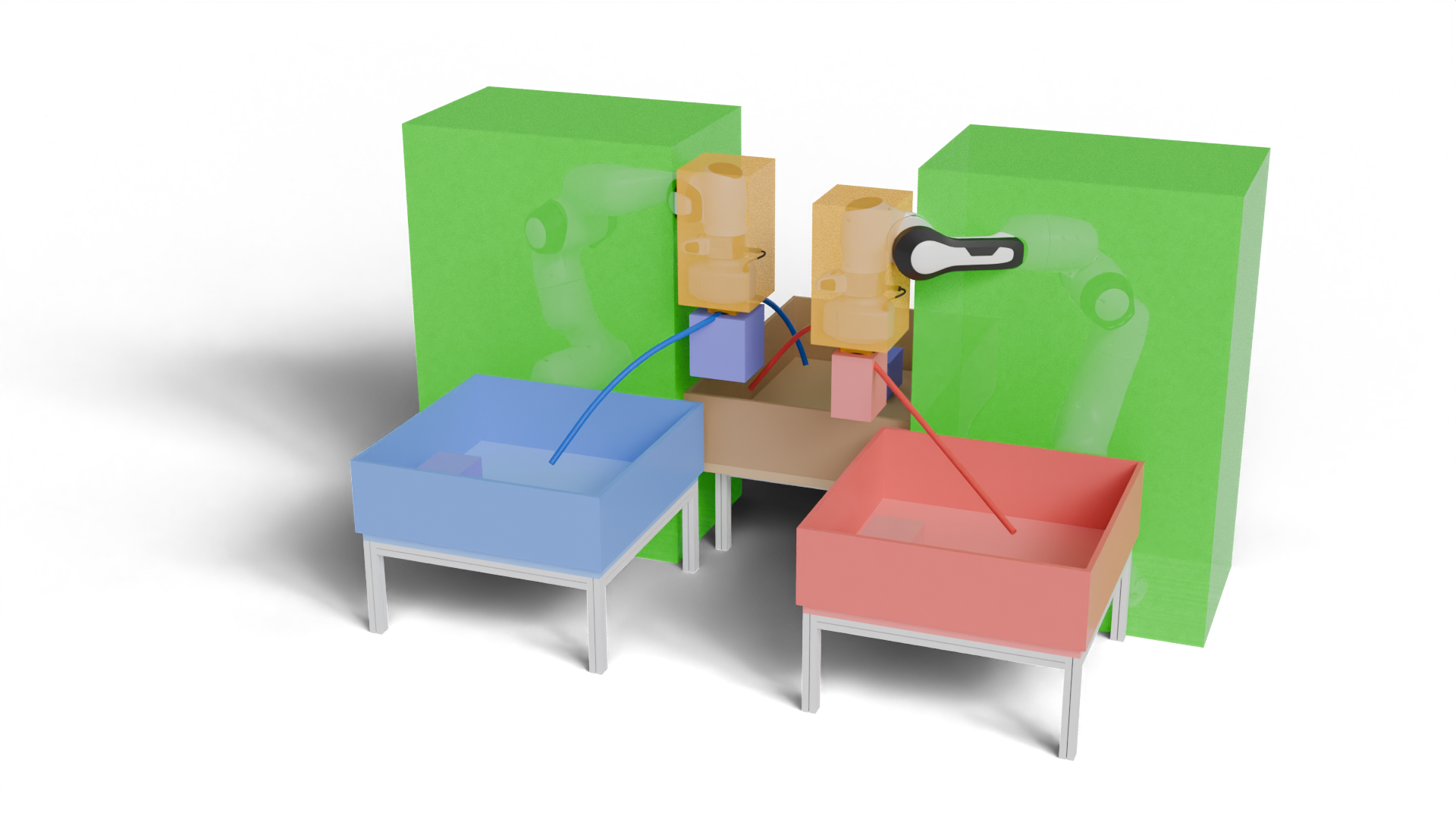}
    
    \caption{Dual-arm bin unloading in task space. From left to right: \textit{(first)} The task: two Franka manipulators move packages from a central bin to their respective offload zones. \textit{(second)} The waypoint-based initialization (straight-line segments) and the optimized minimum-time trajectories with our proposed BMTP approach. \textit{(third)} Visualization of the valid domains for each gripper position. \textit{(fourth)} Visualization of the axis-aligned bounding box collision geometries used during trajectory optimization. }
    \label{fig:taskspace_planning}
\end{figure*}

Next, we evaluate our BMTP approach on a dual-arm bin unloading problem in task space with convex obstacles, shown in~\Cref{fig:taskspace_planning}. These experiments demonstrate the efficacy of our approach on a common warehouse task and show that we achieve comparable performance to the method from~\cite{marcucci2025biconvex} while being able to handle a larger class of problems and exhibiting robustness to bad initialization.  

In the bin unloading task, two Franka Research 3 manipulators cooperatively move three blue and three red packages from a central bin to their corresponding blue and red offload bins. Each offload task consists of six trajectories: three that carry a package pair from the central bin to the offload bins, and three that return to the central bin. The dimensions and the placement of the packages are randomized. We plan in the six-dimensional task space where the first three dimensions represent the position of one gripper, and the last three that of the other. To convert the resulting task space plan to joint space we use differential inverse kinematics. See~\Cref{app:pallet_randomization,app:diffik} for details. 

All collision geometries are modeled as axis-aligned boxes, including the robot bases, gripper bodies, suction cups, packages, and bins. The collision geometry setup is shown in the rightmost panel of~\Cref{fig:taskspace_planning}. Therefore, the obstacles in the six-dimensional planning space correspond to the set of configurations where any pair of the boxes overlap. For a pair of axis-aligned boxes centered at $p_1, p_2\inR^3$ with half-widths $h_1, h_2\inR^3$, the set of configurations where the two boxes overlap can be written in closed form as the following convex polyhedron
$$\calO = \left\{(p_1,p_2)\in \mathbb{R}^6 \mid |p_{2,i}-p_{1,i}| \leq h_{1,i}+h_{2,i},~ i=1,2,3 \right\}.$$ 
As a result, all obstacles in this problem are convex and the planner needs to circumnavigate 88 obstacles when the robots are moving freely and 123 obstacles when they each are carrying a box.

We additionally set strict velocity and acceleration limits of $0.5\,\mathrm{m/s}$ and $1.0\,\mathrm{m/s^2}$, respectively, for each end effector. This corresponds to making the derivative sets Cartesian products of spheres with the corresponding radii.

We compare our approach to two baselines and include the simple waypoint heuristic for reference. The waypoint heuristic uses five straight-line segments to concurrently move the arms between the bins, and is used to warm-start all planners. We obtain the segment timing by using the approach in~\cite[\S VI]{marcucci2025biconvex}. The first baseline directly solves the nonlinear formulation~\eqref{eq:statement} with an off-the-shelf general nonlinear solver by enforcing collision avoidance at a discrete set of points. We provide the details in~\Cref{app:nonlinear_baselines}. The second baseline we compare against is SCSPlanning~\cite{marcucci2025biconvex}, which computes smooth minimum-time trajectories through sequences of convex safe sets which we compute using the edge inflation algorithm~\cite{werner2025superfast}. We call this baseline Edge Inflation + SCSPlanning (EI+SCS) and provide specific implementation details in~\Cref{app:SCSbaseline}. 

For all approaches, we employ the conservative collision checker from~\Cref{app:collision_checking}, which we run on all trajectory segments in parallel. This checker is conservative in the sense that it guarantees that no trajectory is incorrectly classified collision-free.

For this task we optimize degree-six \bez splines, which guarantee the feasibility of the initial polygonal trajectory with velocity and acceleration constraints. We do not enforce continuous acceleration between trajectory segments on any of the approaches, as these constraints are not supported by the EI+SCS baseline. For the BMTP approach we use a plane degree equal to one. We run both EI+SCS and BMTP until the relative change in cost between successive feasible iterates falls below 1\%. For the nonconvex approach with the discretized collision avoidance constraints we use 25 evenly spaced constraints with an influence distance of $1\times10^{-2}\,\mathrm{m}$ and lower bound of $5\times10^{-3}\,\mathrm{m}$ per segment. We employ SNOPT to solve these programs~\cite{gill2005snopt}. To solve the convex subproblems in our BMTP approach and in SCSPlanning we employ the Clarabel solver~\cite{goulart2026clarabel}. The results across 50 random instances (300 trajectory optimizations per approach) are summarized in~\Cref{tab:taskspace_comparison}.

\begin{table}[!h]
    \centering
    \caption{Statistics of Randomized Dual-arm Bin Unloading Tasks (N=300). All results are reported as mean $\pm$ std.}
    \label{tab:taskspace_comparison}
\resizebox{\linewidth}{!}{%
  \begin{tabular}{l|c|ccc}
      \toprule
      & Waypoint & Nonconvex (disc.) & EI+SCS & Ours \\
      \midrule
      Traj.\ duration [s]    & $11.06 \pm 1.24$ & $4.39 \pm 2.50$ & $\mathbf{2.77 \pm 0.24}$ & $2.82 \pm 0.21$ \\
      Computation time [ms]  & $5.3 \pm 0.7$ & $2762.2 \pm 1394.1$ & $204.0 \pm 49.7$ & $\mathbf{187.9 \pm 43.3}$ \\
      Computation time p95 [ms]  & $6.3$ & $5135.3$ & $295.4$ & $\mathbf{269.2}$ \\
      Success rate [\%]      & 100.0 & $95.3$ & $\mathbf{100.0}$ & $\mathbf{100.0}$ \\
      Collision-free [\%]    & 100.0 & $94.8$ & $\mathbf{100.0}$ & $\mathbf{100.0}$ \\
      \bottomrule
  \end{tabular}}

\end{table}

Both BMTP and EI+SCS achieve a $100\%$ success rate and produce collision-free trajectories on every instance. The discretized nonconvex baseline is far slower, taking over $2.7$ seconds of computation on average, and less reliable, failing on $4.7\%$ of instances and producing collisions on $5.2\%$. BMTP achieves trajectory quality comparable to EI+SCS ($2.82$ vs.\ $2.77$ seconds average duration) while computing slightly faster both on average ($188$ vs.\ $204$ ms) and at the 95th percentile ($269$ vs.\ $295$ ms).  Across the $300$ plans, BMTP converged after solving $6.92 \pm 1.31$ trajectory updates on average, of which $4.83 \pm 0.64$ produced an improving collision-free trajectory. The timing breakdown in~\Cref{fig:timing_breakdown} shows that EI+SCS spends about an eighth of its computation on constructing the sequence of convex sets.

The employed waypoint heuristic in the evaluation moves both arms concurrently. A natural simplification is to naively move the arms along the same paths but in sequence rather than concurrently. This results in a much higher cost initialization that consists of 10 trajectory segments instead of 5. We rerun the entire evaluation with this simplified initialization and report the results in \Cref{tab:taskspace_comparison_sequential}. Interestingly, we find that if we initialize our BMTP approach sequentially, the trajectory duration is not meaningfully impacted and ends up improving marginally and additionally beating EI+SCS with the stronger initialization. The cost of the sequential initialization is reflected in the 2.75-times increase of the computation time. With this warm start, it took $10.86 \pm 1.95$ trajectory updates on average for our approach to converge, of which $6.63 \pm 0.76$ produced improving collision-free trajectories.

\begin{figure}[t]
    \centering
    \includegraphics[width=\columnwidth]{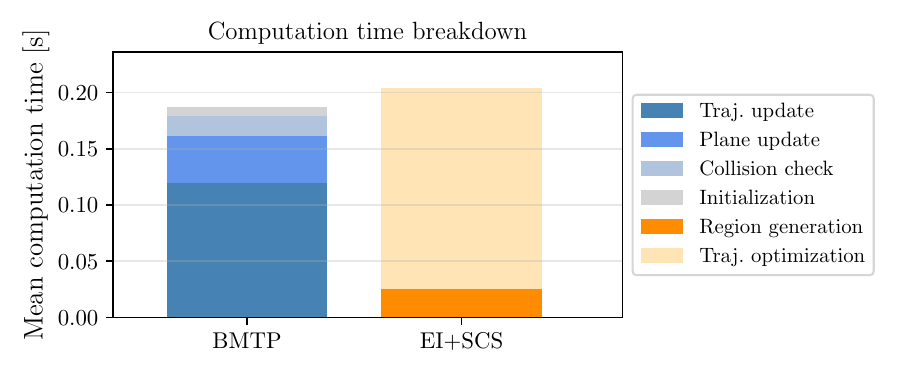}
    \caption{Timing breakdown of our BMTP approach and the EI+SCS baseline. We see that around an eighth of the computation time is spent on constructing the sequence of convex sets.}
    \label{fig:timing_breakdown}
\end{figure}

\begin{table}[!h]
    \centering
    \caption{Statistics of Randomized Dual-arm Bin Unloading Tasks using Sequential Waypoint Planner (N=300). All results are reported as mean $\pm$ std.}
    \label{tab:taskspace_comparison_sequential}
\resizebox{\linewidth}{!}{%
  \begin{tabular}{l|c|cc}
      \toprule
      & Waypoint (sequential) & Ours (sequential) & Ours (concurrent) \\
      \midrule
      Traj.\ duration [s]    & $18.44 \pm 2.42$ &$ \mathbf{2.70 \pm 0.21}$& $2.82 \pm 0.21$  \\
      Computation time [ms]  & $7.5 \pm 0.9$ & $516.6 \pm 149.7$& $\mathbf{187.9 \pm 43.3}$ \\
      Success rate [\%]      & $100.0$ & $100.0$ & $100.0$ \\
      Collision-free [\%]    & $100.0$ & $100.0$ & $100.0$ \\
      \bottomrule
  \end{tabular}}
\end{table}

On this task we compared our BMTP against SCSPlanning, which is designed to solve exactly these kinds of minimum-time problems with velocity and acceleration constraints quickly and reliably. Because SCSPlanning requires a precomputed sequence of safe sets, we pair it with the fast Edge Inflation safe-set construction approach. We found that BMTP matches the trajectory quality and computation time of EI+SCS, while constructing and adapting the necessary separating planes on the fly, without precomputing and fixing the sequence of safe sets. The nonconvex baseline, which enforces collision avoidance at discrete points and solves the resulting nonlinear program with a general solver, is substantially slower and does not reliably return collision-free trajectories. We explored enforcing collision avoidance through the bilinear formulation~\eqref{eqn:biconvex_problem} and directly solving the program with a nonlinear solver, but this led to frequent solver failures and would not solve reliably. Finally, we observe that BMTP is capable of finding high-quality trajectories despite weaker warm starts. Using a far longer initial path leaves the resulting optimized trajectory essentially unchanged at a slight increase of computational cost. This is a structural advantage over methods that optimize within a fixed convex decomposition which is locally created around an initial feasible path, as the quality of the final trajectory is inherently limited by the quality of the safe sets.


\subsection{Hardware validation}\label{ssec:Hardware_validation}

In this section, we validate our simulation results from the bin unloading task in~\Cref{ssec:pallet_unloading} on real hardware. To match our available hardware, we apply some slight modifications to the experimental setup. We use two Franka Research 3 robot manipulators with the standard Franka grippers equipped with 3D-printed Fin Ray fingers which have a strip of 3M TB641 grip tape on the grasping surface. Due to the increased size of the grippers, we only consider unloading two red and two gray bricks in two concurrent, dual-arm motions.

The system has no perception, so we fix the brick start positions to a 2$\times$2 grid and permute their arrangement. We find that 5 of the 6 possible permutations can be unloaded in two concurrent, dual-arm motions. To avoid collisions, we again employ axis-aligned boxes as shown in~\Cref{fig:hardware_setup}. This setup has 120 obstacles in the planning space when the arms are each holding a brick, and 93 when they are moving freely. For the different arrangements we hard-code the sequence of actions, e.g., which target positions to plan to and when to open and close the gripper, but perform the planning on the fly after every motion of the arm. To avoid placing the bricks manually on the designated spots, we have the robots first stage the brick-arrangement by executing the action sequence in reverse and then unloading it. As such, the robots can autonomously cycle through a full evaluation of a planning approach. Each cycle of staging and unloading the bin consists of 10 motion plans; therefore, the entire evaluation consists of planning and executing 50 motion plans.

As in the simulation experiments, we plan in the six-dimensional task space of the two end-effector positions and convert to configuration space with differential inverse kinematics, see~\Cref{app:diffik}. We reuse the same planner settings (degree-6 trajectory splines, degree-1 planes, and a 1\% convergence threshold), tightening only the velocity and acceleration limits to $0.4\,\mathrm{m/s}$ and $0.8\,\mathrm{m/s^2}$ per end-effector. Each arm is driven by an independent instance of our open-source Franka driver~\cite{werner2026franka}, which tracks the streamed trajectory with a $1\,\mathrm{kHz}$ joint-impedance controller (see \Cref{app:armdriver}). We compare against the same Waypoint and EI+SCS baselines.

We summarize the results for the 50 trajectories per approach in~\Cref{tab:hardware_validation}.~\Cref{fig:hw_unloading_seq} shows a representative unloading sequence that was planned with our BMTP approach. All three planners achieve $100\%$ planning success and $100\%$ collision-free execution. BMTP and EI+SCS perform similarly on the measured metrics on average: end-to-end wall time per 10-phase cycle is $38.96$\,s for BMTP vs $39.60$\,s for EI+SCS, and average planning time per stage is $145.1$\,ms vs $150.2$\,ms. The per-phase computation time breakdown is shown in~\Cref{fig:timing_breakdown_hw}. Across the $50$ planning phases, the BMTP approach converged after solving $7.32 \pm 1.50$ trajectory updates on average (range $5$--$11$), of which $5.72 \pm 0.83$ produced feasible and improving trajectories. We measured a root-mean-squared tracking error for each end-effector of less than 3 mm for all three approaches. We observe that our method computes equal-quality trajectories in roughly the same time as EI+SCS. Both approaches execute roughly four times faster than the Waypoint baseline. On one hardware plan, the Waypoint planner tripped a driver-side safeguard that further slowed execution to keep joint velocities within the hardware limits.

\begin{figure*}[t]
    \centering
    \includegraphics[width=0.48\textwidth]{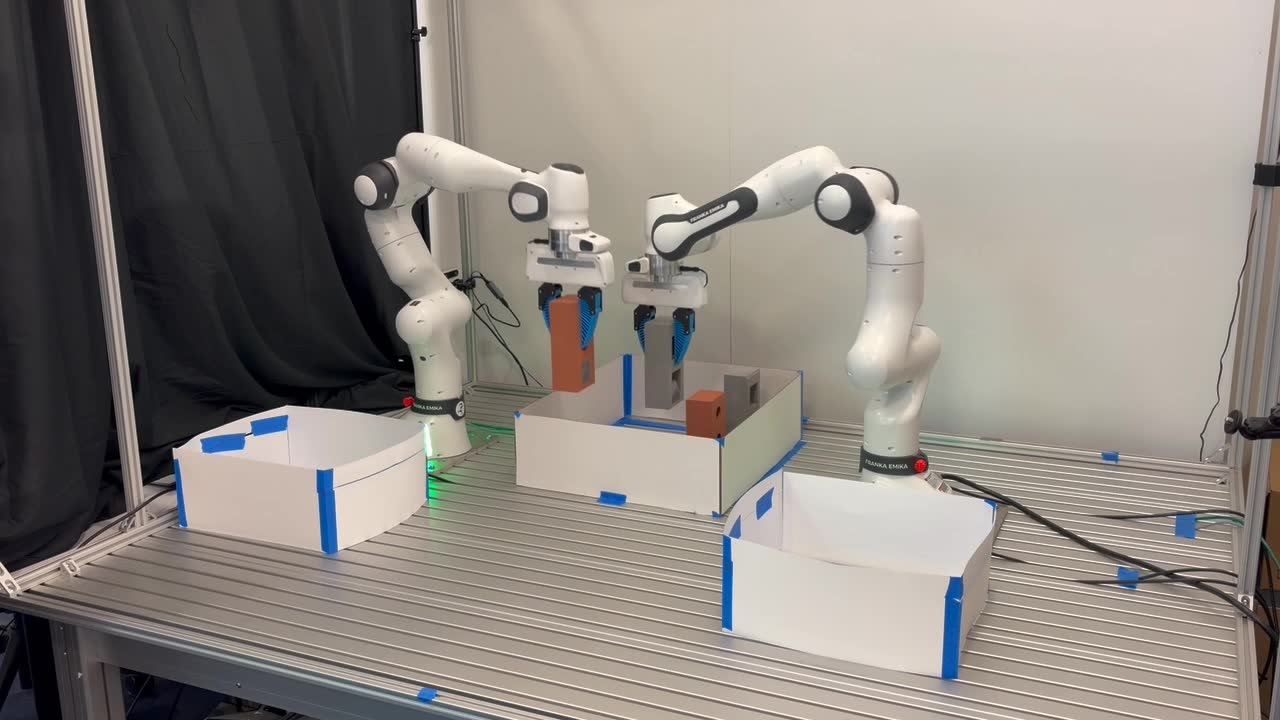}
    \hfill
    \includegraphics[width=0.48\textwidth]{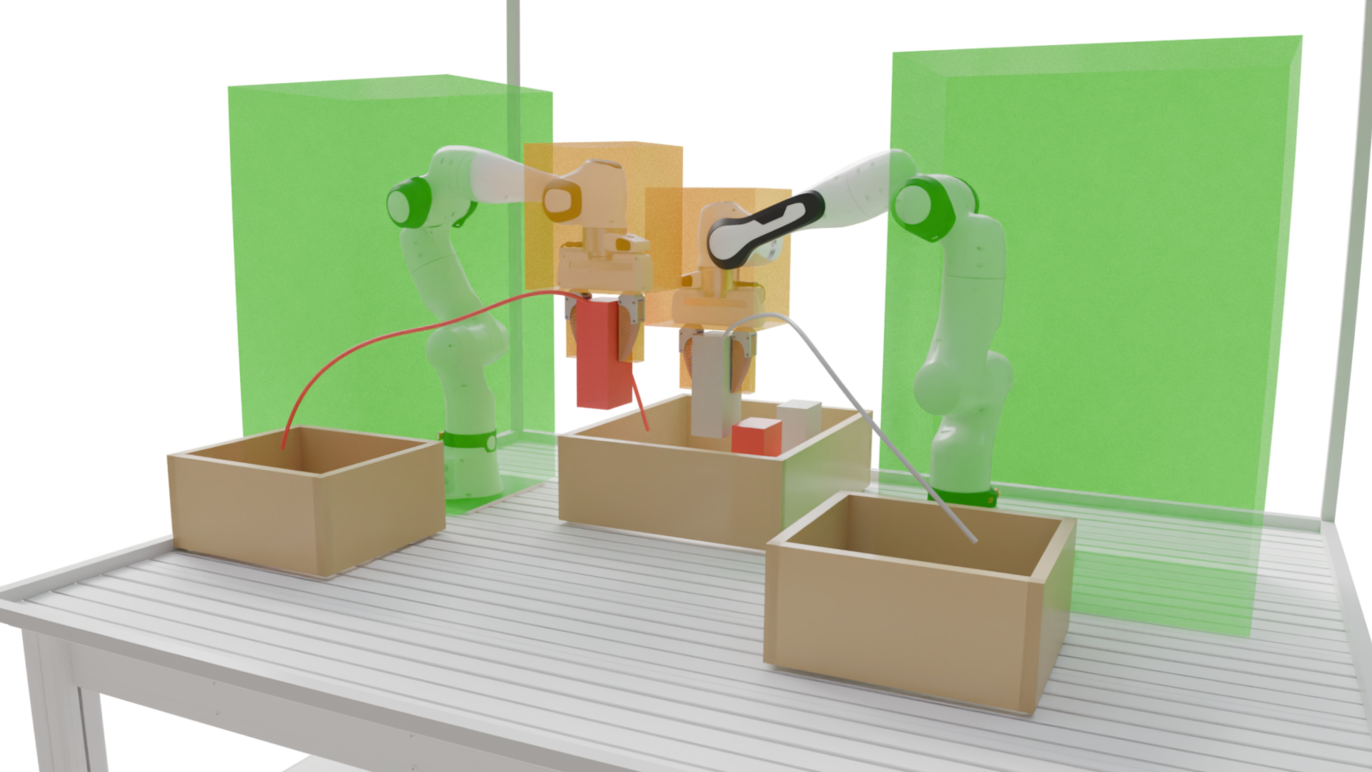}
    \caption{Unloading task used to validate BMTP. Two Franka Research 3 manipulators are employed to unload four bricks from a common bin. The planner operates in the six-dimensional task space given by the two end-effector positions; the resulting trajectories are converted to configuration space using differential inverse kinematics. \textit{(left)} Hardware setup. \textit{(right)} Visualization of the modeled system along with the axis-aligned box collision geometries used for planning. The current trajectory being executed is shown in red and gray. At every point along the trajectory, our planner continuously guarantees that all constraints are satisfied and that the collision geometries do not overlap. }
    \label{fig:hardware_setup}
\end{figure*}

\newcommand{\seqfig}[1]{\includegraphics[width=0.19\textwidth]{#1}}
\begin{figure*}[t]
    \centering
    \seqfig{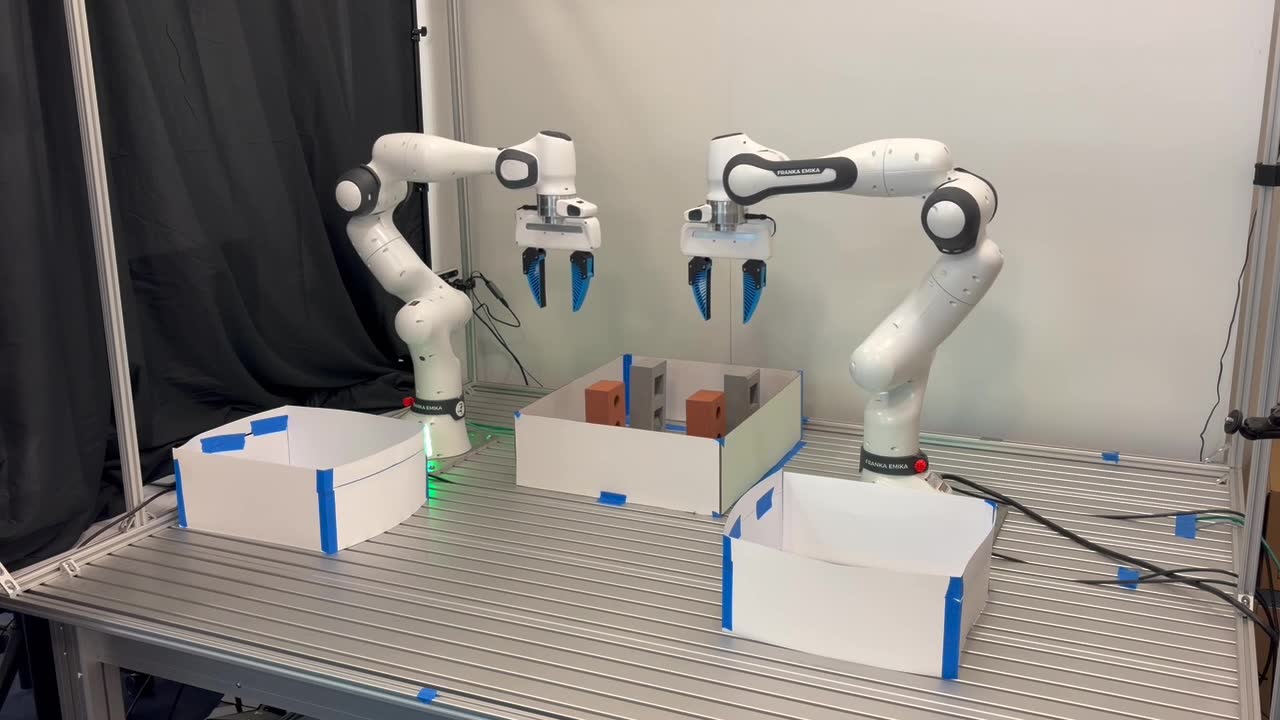}\hfill
    \seqfig{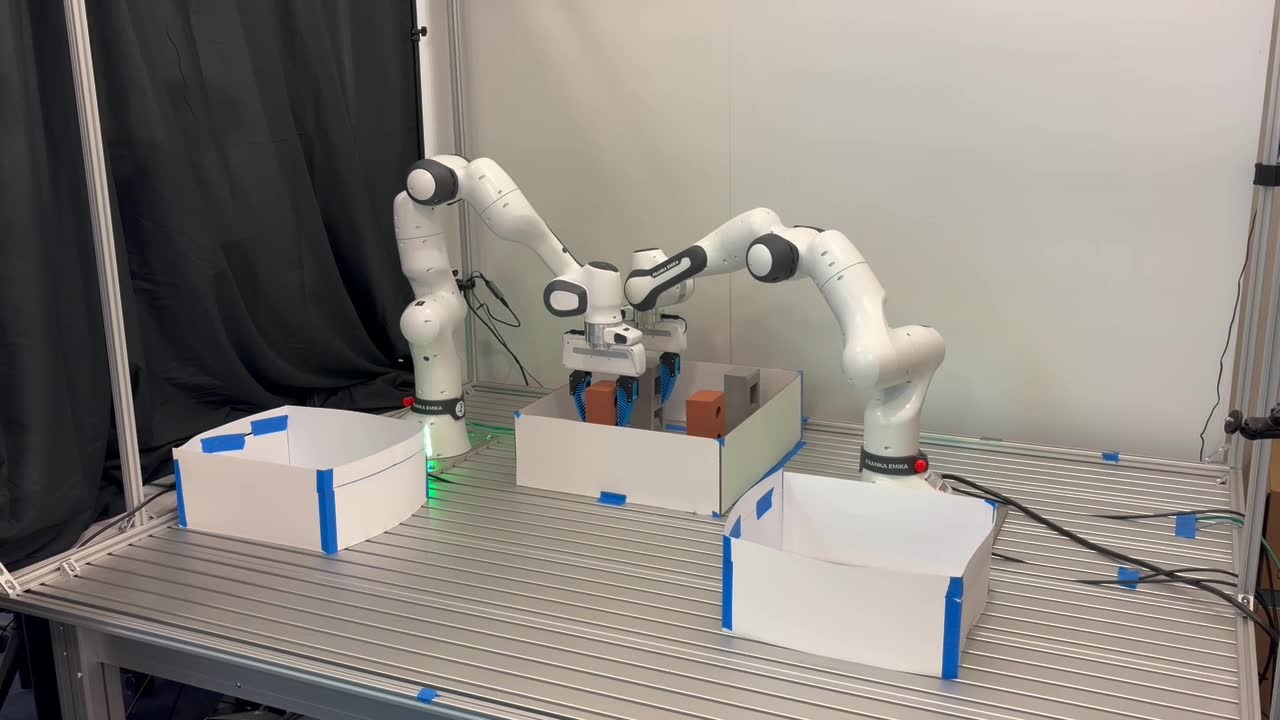}\hfill
    \seqfig{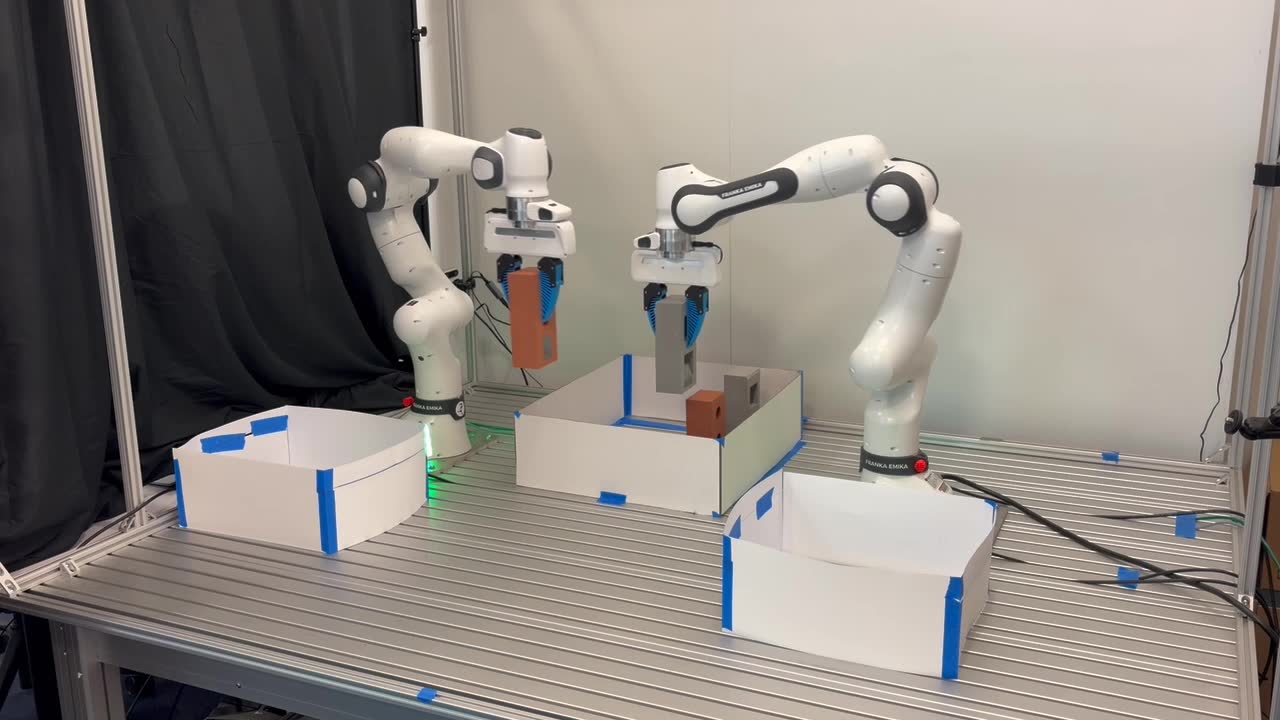}\hfill
    \seqfig{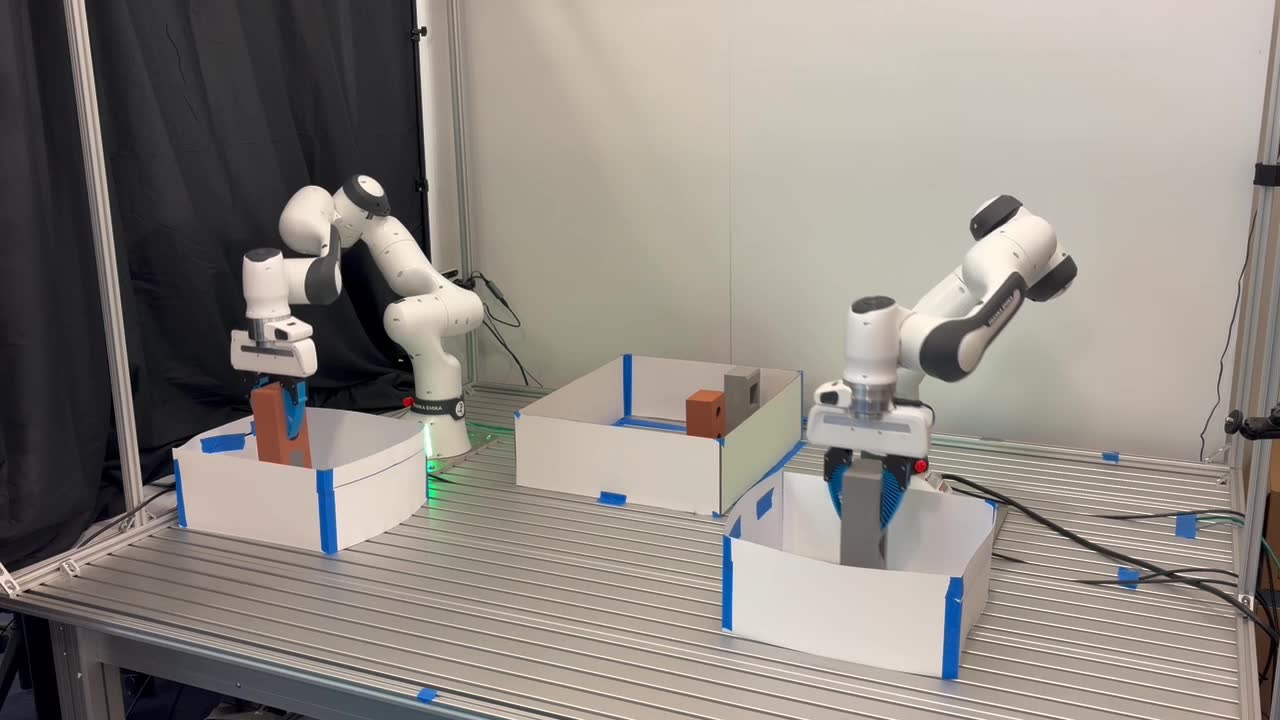}\hfill
    \seqfig{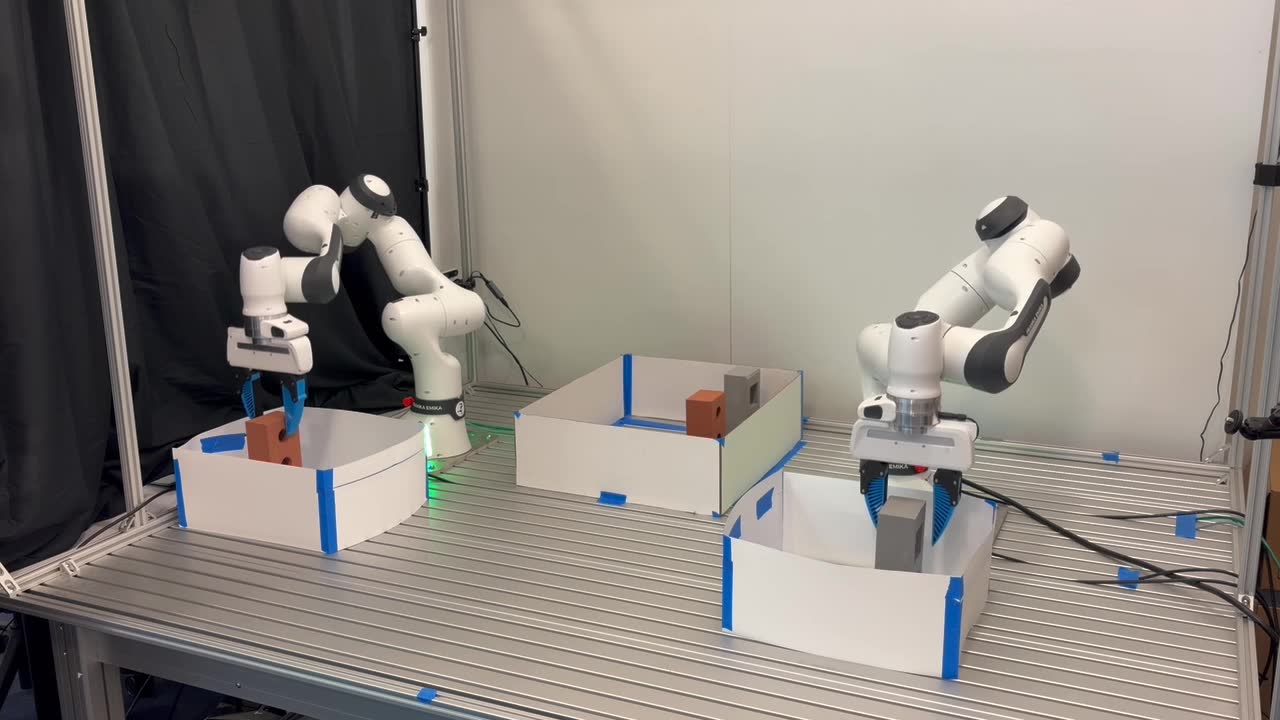}\\[4pt]
    \seqfig{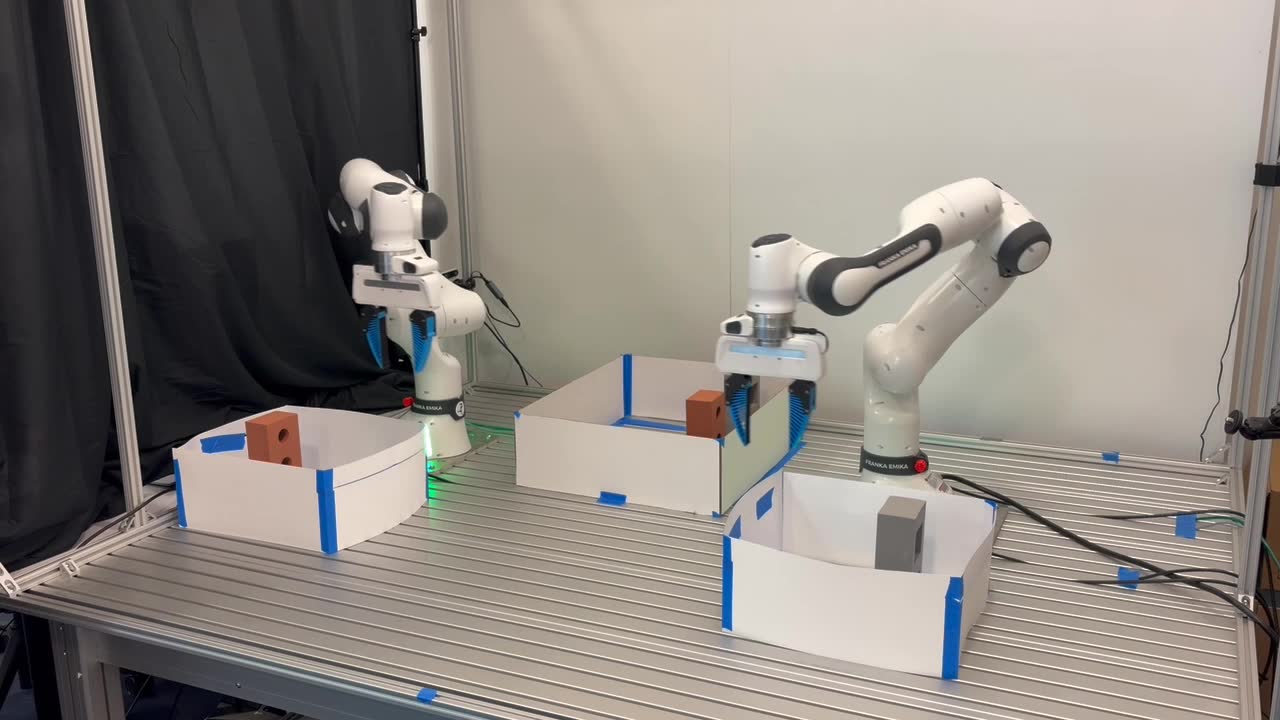}\hfill
    \seqfig{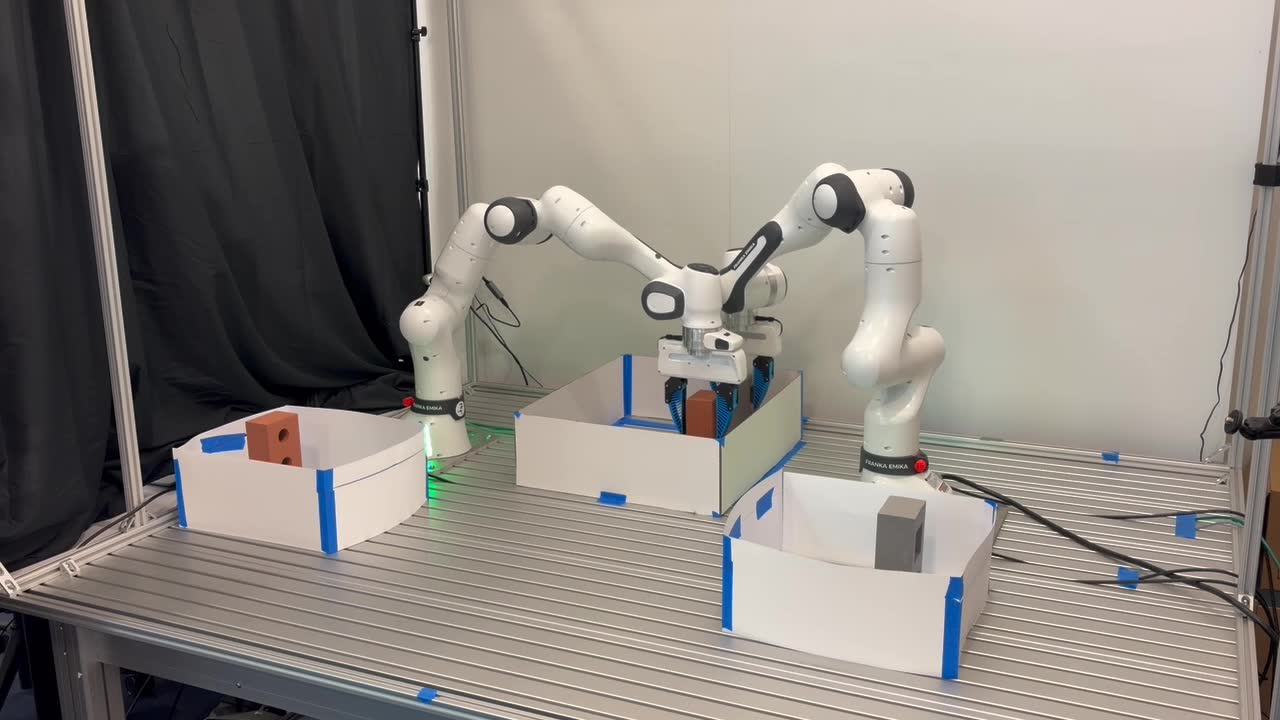}\hfill
    \seqfig{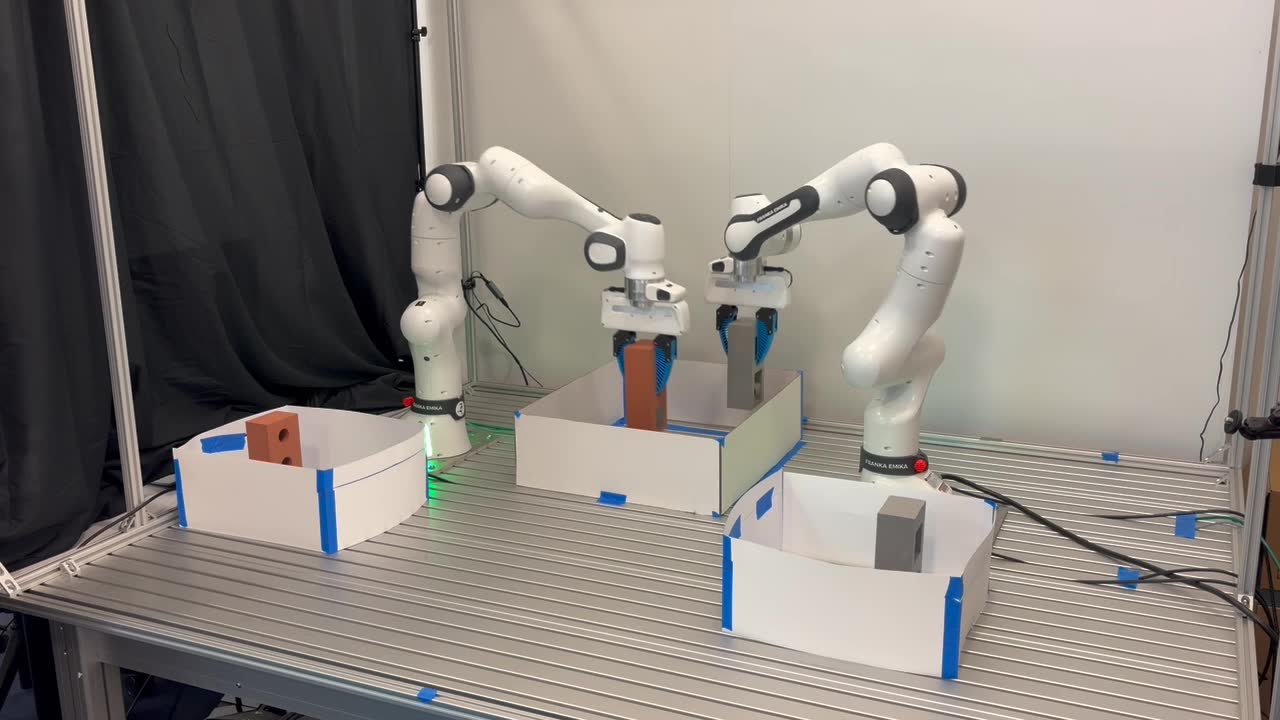}\hfill
    \seqfig{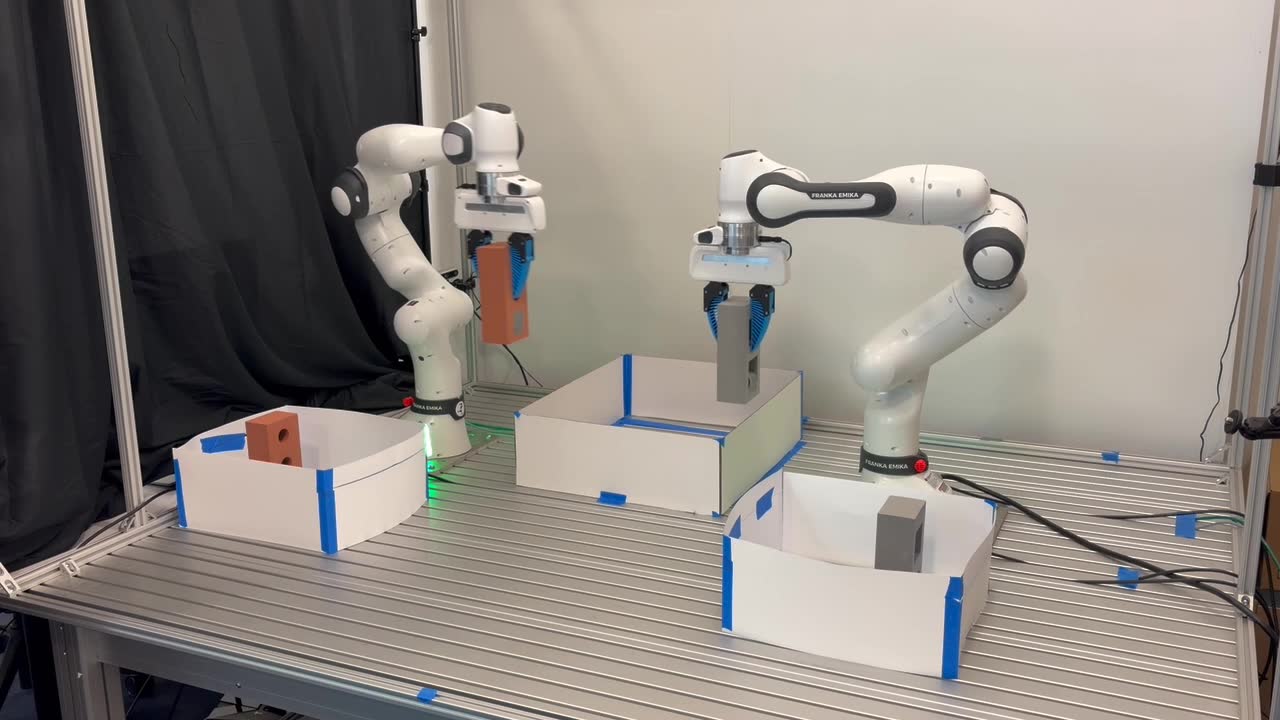}\hfill
    \seqfig{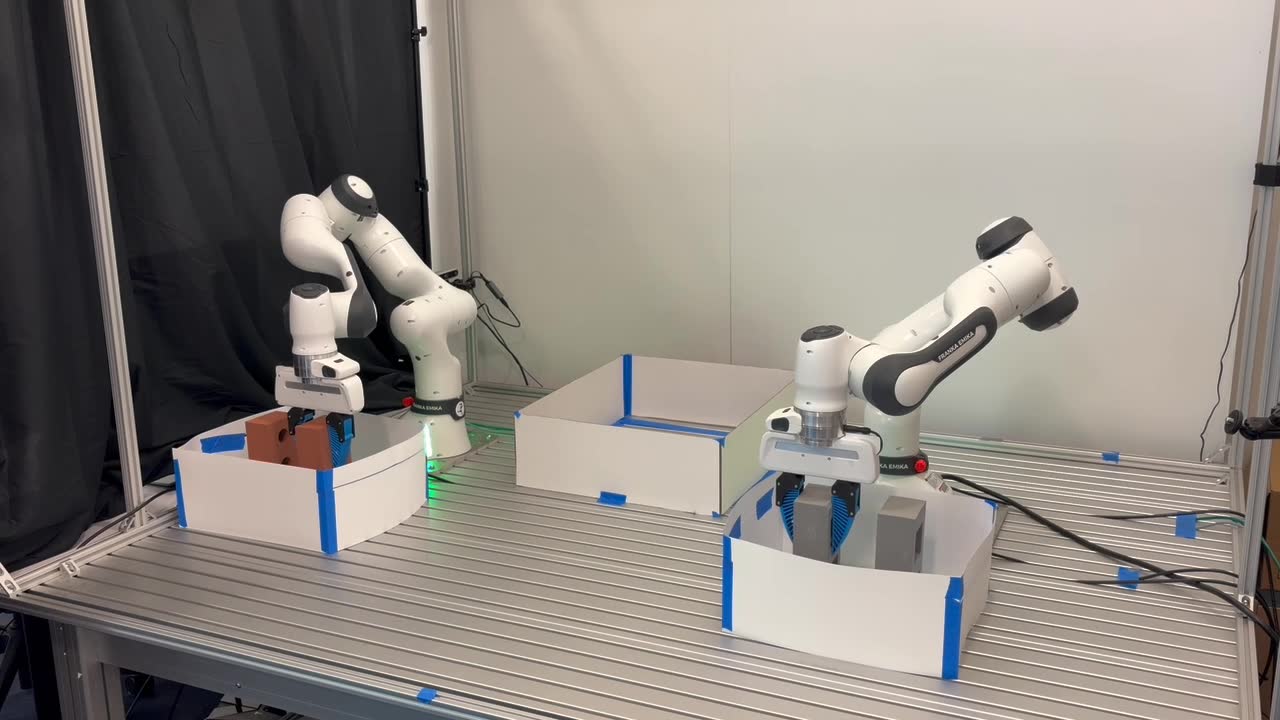}
    \caption{Hardware execution of the bin unloading problem with our BMTP approach. The two robot manipulators transfer all four bricks from the central bin to the corresponding offload bins: the red bricks to the left bin and the gray bricks to the right. Frames are shown in chronological order, left-to-right, top-to-bottom.}
    \label{fig:hw_unloading_seq}
\end{figure*}

\input{sections/hw_validation_table.tex}

\begin{figure}[t]
    \centering
    \includegraphics[width=\columnwidth]{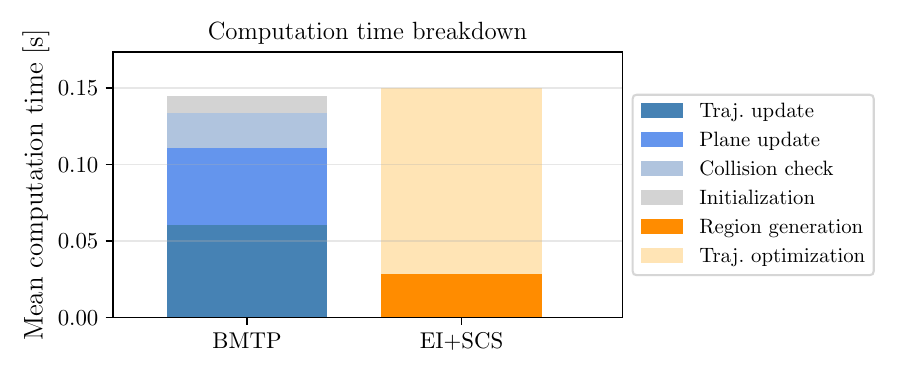}
    \caption{Mean computation time breakdown for the hardware experiments.
    }
    \label{fig:timing_breakdown_hw}
    \vspace{-0.1cm}
\end{figure}

%% file: sections/hw_validation_table.tex
\begin{table}[!h]
    \centering
    \caption{Hardware results for all five Stage \& Unload cycles. All results are reported as mean $\pm$ std.}
    \label{tab:hardware_validation}
    \resizebox{\linewidth}{!}{%
        \begin{tabular}{l|c|cc}
            \toprule
            & Waypoint & EI+SCS & BMTP (ours) \\
            \midrule
            Aggregate plan duration [s] & 131.14 $\pm$ 11.89 & 32.75 $\pm$ 1.70 & $\mathbf{32.24 \pm 1.82}$ \\
            Avg plan time per stage [ms] & $3.0 \pm 0.1$ & 150.2 $\pm$ 8.0 & $\mathbf{145.1 \pm 10.9}$ \\
            End-to-end wall time [s] & 139.53 $\pm$ 12.17 & 39.60 $\pm$ 1.69 & $\mathbf{38.96 \pm 2.01}$ \\
            Success rate [\%] & 100 & 100 & 100 \\
            Collision-free [\%] & 100 & 100 & 100 \\
            \bottomrule
        \end{tabular}}
\end{table}


%% file: sections/discussion.tex
\section{Discussion}\label{sec:discussion}
In this section we briefly discuss our key findings and some limitations of our BMTP approach.



\paragraph{Performance}
On the bin unloading task, BMTP matches the trajectory quality of EI+SCS and is slightly faster to compute, both in simulation and on hardware~(\Cref{ssec:pallet_unloading,ssec:Hardware_validation}). EI+SCS builds on SCSPlanning~\cite{marcucci2025biconvex}, a fast decomposition-based planner used in warehouse applications, which we pair with fast on-the-fly convex-set generation~\cite{werner2025superfast}. Besides being performant in terms of raw computation time and trajectory cost, BMTP does not require a precomputed convex decomposition of the free space, which needs to be recomputed whenever the environment changes. And as seen in the village example~(\Cref{ssec:village}), our BMTP can significantly shorten trajectories by jumping around obstacles. This is not possible for an approach like EI+SCS, and other methods that build and constrain the resulting trajectory to be contained in a sequence of convex sets around an initial trajectory~\cite{chen2016online, liu2017planning, wu2024optimal}. 

\paragraph{Generality}
Beyond the comparable performance to EI+SCS, BMTP supports a larger problem class. In the village experiment~(\Cref{ssec:village}), we jointly optimize the time and shape of a snap-continuous trajectory with bounded velocity, acceleration, jerk, and snap. In order to support this combination of costs and constraints, the method in~\cite{marcucci2024fast} requires a combination of a linearization step and a trust region, leading to a planning algorithm that is significantly slower and highly sensitive to the chosen trust-region parameters.

\paragraph{Robustness to bad initialization}
We also observe in the bin unloading experiment that BMTP exhibits some robustness to poor initialization~(\Cref{ssec:pallet_unloading}). Our BMTP ended up producing the overall lowest-cost trajectories on average when using the much weaker initialization.

\paragraph{Limitations}
Our approach also has several limitations. First, it currently supports only convex obstacles. Extending it to nonconvex obstacles, and ultimately to planning directly in configuration space, is an important direction for future work, but poses a substantial challenge. Although the polar is convex even for a nonconvex obstacle, the resulting planes separate the trajectory from the obstacle's convex hull. In many practical settings, such as configuration-space planning for a robot manipulator, this is overly conservative. It thus remains open whether the trajectory can be separated locally from nonconvex obstacles efficiently. Second, although the feasible iterates of BMTP never increase the objective cost of~\eqref{eqn:biconvex_problem} and the procedure converges to high-quality trajectories in practice, the point it converges to need not be a local optimum of~\eqref{eqn:biconvex_problem}. Third, it requires a collision-free path as initialization and does not itself address finding an initial feasible path. In practice, we found either a simple waypoint planner or a sampling-based planner such as RRT~\cite{kuffner2000rrt} to be effective for this purpose. Furthermore, we do not have a principled way to select the degree of the separating planes, although a plane degree of one worked well across all our experiments. Finally, while the trajectory update formulation allows derivative constraints of arbitrary degree, the numerics can become challenging beyond snap.

%% file: sections/software.tex
\section{Software}\label{sec:software}
We provide a pip-installable, open-source implementation of our planner in Python along with this paper. The software, PyBMTP, is available at \url{https://github.com/wernerpe/pybmtp}. It uses Drake~\cite{drake} to parse the convex subproblems, and Clarabel to solve them~\cite{goulart2026clarabel}.

Given a sequence of waypoints, which construct a collision-free polygonal trajectory between the start and the goal, a list of convex obstacles, the admissible domain, and convex velocity and acceleration sets, planning a minimum-time
collision-free \bez trajectory takes only a few lines:

\begin{lstlisting}[language=Python]
import numpy as np
import pydrake.all as pd
from pybmtp import solve_minimum_time, Limits

# velocity/acceleration limits
z = np.zeros(dim)
S = pd.Hyperellipsoid.MakeHypersphere
limits = Limits(velocity=S(v_max, z),
                acceleration=S(a_max, z))

res = solve_minimum_time(
    waypoints, obstacles, domain, limits)

traj = res.trajectory  # min-time curve
T = res.total_time     # duration [s]
\end{lstlisting}

%% file: sections/conclusion.tex
\section{Conclusion}\label{sec:conclusion}
We have presented BMTP, a biconvex minimum-time planner which can effectively plan around convex obstacles, supports derivative and continuity constraints to an arbitrary degree, and does not require a costly convex decomposition of the free space up front. We jointly convexify the minimum-time objective and the derivative constraints through a change of variables, and reformulate the collision-avoidance constraints as a search for time-varying separating planes, transforming the problem to a biconvex program. Our biconvex procedure refines an initial collision-free path by alternating between computing maximum-margin separating planes and re-optimizing the trajectory, and only tags obstacles that the current iterate collides with, which allows the trajectory to jump around obstacles and escape some bad local solutions. The procedure can be initialized with a collision-free path as a simple polygonal curve from a sampling-based planner, and never increases the trajectory duration across iterations. 

We demonstrated the generality of our approach by planning a snap-continuous minimum-time trajectory for a quadrotor through a cluttered village environment, a combination of costs and constraints that, to the best of our knowledge, no existing decomposition-based planner supports out of the box. On a dual-arm bin unloading benchmark, our approach matches the quality and computation time of a state-of-the-art decomposition-based planner while being more general and more robust to bad initialization. Additionally, we validated our approach on hardware with two Franka manipulators. In future work, we plan to extend the approach to nonconvex obstacles and to planning directly in configuration space, and to leverage hardware acceleration to further improve performance.

%% file: sections/acknowledgements.tex
\section*{Acknowledgments}
We are very grateful for the funding provided by the Office of Naval Research, Award Number N00014-23-1-2354. Research was also sponsored in part by the Department of the Air Force Artificial Intelligence Accelerator and was accomplished under Cooperative Agreement Number FA8750-19-2-1000. The views and conclusions contained in this document are those of the authors and should not be interpreted as representing the official policies, either expressed or implied, of the Department of the Air Force or the U.S. Government. The U.S. Government is authorized to reproduce and distribute reprints for Government purposes notwithstanding any copyright notation herein. We further thank Sam Creasey for help with the arm drivers, Russ Tedrake for guidance in the early stages of the project, and Nicholas Pfaff for help with the grippers. The rendered figures in this paper were produced using Drake Blender Tools~\cite{pfaff2025drakeblender}.

%% file: sections/appendix.tex
\subsection{Nonconvexities in DBMP constraint enforcement}\label{app:dbmp_nonconvexity}
Higher-order derivative and continuity constraints are often crucial in practice, but they are hard to enforce jointly with a minimum-time objective. The core difficulty already appears for a single segment. After the normalization of~\Cref{sec:minimum_time}, the $i$th derivative constraint on a trajectory of duration $T$ reads
$$r^{(i)}(s)\in T^i\,\mathcal C_i,\quad i=1,\dots,I,\ s\in[0,1].$$
Even for convex $\mathcal C_i$, scaling by $T^i$ makes this nonconvex in $(r^{(i)}(s), T)$ for $i\geq 2$. Our change of variables $T_I=T^I$ removes this obstruction, but only by acting on a single, shared duration~(\Cref{sec:minimum_time}).

Multiple segments further complicate things. DBMPs give each segment $m$ its own duration $T_m$, so the change of variables must be applied per segment. This breaks the objective, and the continuity constraints then couple the durations nonlinearly. Matching the $i$th physical derivative at the junction of segments $m$ and $m+1$ requires
$$\frac{1}{T_m^i}\,r_m^{(i)}(\tfrac{m}{M})=\frac{1}{T_{m+1}^i}\,r_{m+1}^{(i)}(\tfrac{m}{M}),$$
a nonlinear equality. A single shared duration instead cancels the $\tfrac{1}{T^i}$ factors, reducing continuity to the linear constraint $r_m^{(i)}(\tfrac{m}{M})=r_{m+1}^{(i)}(\tfrac{m}{M})$.

\subsection{A recipe for computing the polar of convex sets in conic standard form}\label{app:conic_polar}

Let $\calO = \{x\inR^n \mid \exists y\inR^m : Ex + Fy + g \in \calK\}$ be a convex set described in conic standard form, where $E\inR^{p\times n}$, $F\inR^{p\times m}$, $g\inR^{p}$, and $\calK\subseteq\R^{p}$ is a closed convex cone with dual cone $\calK^*$ (see~\cite[\S2.6]{boyd2004convex}). The homogenization of $\calO$ reads
$$\tilde\calO = \{(z, t)\inR^{n+1} \mid \exists y : Ez + Fy + gt \in \calK,~ t \geq 0\},$$
and the polar $\calO^\circ$ is its dual cone $\tilde\calO^*$, see~\cite[\S2.1.2]{marcucci2024graphs}. We have $(a,b)\in\calO^\circ$ if and only if $$a^\top z + b t \geq 0~~\text{for all}~(z,t)\in\tilde\calO.$$

Therefore, the element $(a, b)$ is in the polar if the minimum of $a^\top z + b t$ subject to $Ez + Fy + gt \in \calK$, $t\geq 0$ is nonnegative. The Lagrangian of this minimization reads
$$\calL = z^\top(a - E^\top\nu) - y^\top F^\top\nu + t(b - g^\top\nu - \theta),$$
with $\nu\in\calK^*$ and $\theta\geq 0$. For the minimum over the primal variables $(z, y, t)$ to be finite, we require $a = E^\top\nu$, $F^\top\nu = 0$, and $b \geq g^\top\nu$ (absorbing $\theta$). The polar is therefore
\begin{equation}\label{eqn:conic_polar}
\calO^\circ = \left\{(a, b)\inR^{n+1} ~\middle|~
\begin{array}{l}
a = E^\top\nu,~ F^\top\nu = 0,\\
b \geq g^\top\nu,~ \nu \in \calK^*
\end{array}
\right\}.
\end{equation}
To formulate obstacle avoidance constraints for a given $\calO$, we can simply add a separating plane $(a,b)$ by introducing auxiliary decision variables $\nu$ and adding the constraints on the decision variables $(a, b, \nu)$ in~\eqref{eqn:conic_polar}.

In practice, it is useful to include a small step-back from the obstacles. Adding a positive constant $\varepsilon$ to the bound, $b \geq g^\top \nu + \varepsilon$, tightens the polar so that every resulting separating plane with nonzero magnitude keeps a small margin from the obstacle.



\subsection{Conservative collision checking for \bez curves}\label{app:collision_checking}
We check whether a \bez curve $B(s) = \sum_{d=0}^D \beta_d(s)\pi_d$ avoids an H-polyhedron obstacle $\calO = \{x\inR^n \mid Cx \leq d\}$ using the recursive subdivision procedure in~\Cref{alg:collision_check}. The algorithm is conservative: it never falsely certifies a colliding curve as collision-free. The key insight is that by~\Cref{prop:cvx_hull}, if all control points lie outside a single halfspace of $\calO$, the entire curve is guaranteed to avoid $\calO$. When this test is inconclusive, we split the curve in half via De Casteljau subdivision~\cite[\S2.4]{farouki1988algorithms}, and recursively run the same check on the two new pieces. This subdivision tightens the convex hull around the curve, making the test increasingly precise. We terminate if either an endpoint of the curve is inside an obstacle, we find a single halfspace of the obstacle that contains none of the control points, or the segment is so small that all of its control points fit in a user-specified $\varepsilon$-ball. Although we describe the check for H-polyhedron obstacles, it extends directly to any convex obstacle that admits a supporting-hyperplane oracle, such as a sphere, by replacing the per-facet halfspace test with a single separating-plane test.
\begin{algorithm}
\SetAlgoLined
\caption{\textsc{CollisionFree}$(B, C, d, \varepsilon)$}
\label{alg:collision_check}
\SetKwInput{Input}{Input}
\SetKwInput{Output}{Output}
\Input{\bez curve $B$ with control points $\pi_0,\ldots,\pi_D$,\\ obstacle $\calO = \{x \mid Cx\leq d\}$, tolerance $\epsilon$}
\Output{\texttt{true} if $B$ is certified collision-free, \texttt{false} otherwise}
\If{$\pi_0 \in \calO$ \textbf{or} $\pi_D\in\calO$}{
    \Return \texttt{false}\tcp*{endpoint in collision}
}
\If{$\exists\, j$ s.t.\ $c_j^\top \pi_d > d_j$ for all $d = 0,\ldots, D$}{
    \Return \texttt{true}\tcp*{all ctrl pts outside halfspace $j$}
}
\If{bounding radius of $\{\pi_0,\ldots,\pi_D\} \leq \varepsilon$}{
    \Return \texttt{false}\tcp*{conservative: tolerance reached}
}
Split $B$ at midpoint into $B_L, B_R$\\
\Return \textsc{CollisionFree}$(B_L, C, d, \varepsilon)$ \textbf{and} \textsc{CollisionFree}$(B_R, C, d, \varepsilon)$
\end{algorithm}

\subsection{Comments on the maximum-margin planes}\label{app:max_margin_plane_dual}
Finding the maximum-margin separating planes corresponds to solving
\begin{subequations}
\label{eqn:max_margin_planes}
\begin{align}
\text{minimize} \quad & \int_{\calS_k} a_k(s)^\top r(s) + b_k(s)\,ds,\label{eqn:max_margin_planes_cost}\\
\text{subject to} \quad
&  (a_k(s), b_k(s)) \in \calO^\circ_k, \label{eqn:max_margin_planes_polar_cons}\\
& \| a_k \|_2 \leq 1\label{eqn:max_margin_planes_norm_cons},\\
& a_k(s)^\top r(s) + b_k(s)<0. \label{eqn:max_margin_planes_valid_cons}
\end{align}
\end{subequations}
In this section we derive the dual of~\eqref{eqn:max_margin_planes} and verify that an optimal solution to this primal-dual pair corresponds to the one stated in~\eqref{eqn:continuous_max_margin_planes_solution}.

First, we observe that the constraint~\eqref{eqn:max_margin_planes_valid_cons} is redundant since there are no continuity constraints on $(a,b)$. We then introduce auxiliary variables $z_0(s) = 1$ and $z_1(s) = a(s)$ to separate the norm constraint from the decision variable:
\begin{subequations}
\begin{align}
\text{minimize} \quad & \int_{\calS_k} a(s)^\top r(s) + b(s)\, ds, \\
\text{subject to} \quad & (a(s), b(s)) \in \calO^\circ, \\
& \|z_1(s)\|_2 \leq z_0(s), \\
& z_0(s) = 1, \\
& a(s) = z_1(s).
\end{align}
\end{subequations}

We can now write the Lagrangian
\begin{align}
\calL &= \int_{\calS_k} \Big[ a(s)^\top r(s) + b(s) \notag \\
& - \begin{bmatrix} \Xi_a(s)^\top & \Xi_b(s) \end{bmatrix} \begin{bmatrix} a(s) \\ b(s) \end{bmatrix} \notag \\
& - \begin{bmatrix} \Pi_0(s) & \Pi_1(s)^\top \end{bmatrix} \begin{bmatrix} z_0(s) \\ z_1(s) \end{bmatrix} + \nu(s)(z_0(s) - 1) \notag \\
& + \Gamma(s)^\top(a(s) - z_1(s)) \Big] ds
\end{align}
with dual variables $(\Xi_a(s), \Xi_b(s)) \in \tilde\calO$, $(\Pi_0(s), \Pi_1(s)) \in \calL_2$, $\nu(s) \in \mathbb{R}$, and $\Gamma(s) \in \mathbb{R}^n$. Here we used the fact that the dual of the polar $\calO^\circ$ is the homogenization $\tilde\calO$ of the obstacle~\cite[\S2.4.1]{marcucci2024graphs}, and the fact that the second-order cone $\calL_2$ is self-dual. Rearranging terms gives
\begin{align}
\calL &= \int_{\calS_k} \Big[ a(s)^\top(r(s) - \Xi_a(s) + \Gamma(s)) \notag \\
& + b(s)(1 - \Xi_b(s)) + z_1(s)^\top(\Pi_1(s) - \Gamma(s)) \notag \\
& + z_0(s)(\nu(s) - \Pi_0(s)) - \nu(s) \Big]\, ds.
\end{align}

For the minimization of $\calL$ over the primal variables to have a finite value, sufficient conditions for all $s \in \calS_k$ are
\begin{align}
\Xi_a(s) &= r(s) + \Gamma(s), \label{eqn:max_margin_stat_xia}\\
\Xi_b(s) &= 1, \label{eqn:max_margin_stat_xib}\\
\Pi_1(s) &= \Gamma(s), \label{eqn:max_margin_stat_pi1}\\
\nu(s) &= \Pi_0(s). \label{eqn:max_margin_stat_nu}
\end{align}

The dual is therefore
\begin{subequations}
\label{eqn:max_margin_planes_dual}
\begin{align}
\text{maximize} \quad & \int_{\calS_k} -\nu(s)\, ds, \\
\text{subject to} \quad & \|\Gamma(s)\|_2 \leq \nu(s), \\
& r(s) + \Gamma(s) \in \calO_k,
\end{align}
\end{subequations}
where the first constraint follows from $(\Pi_0, \Pi_1) \in \calL_2$ together with~\eqref{eqn:max_margin_stat_pi1} and~\eqref{eqn:max_margin_stat_nu}, and the second follows from $(\Xi_a, \Xi_b) \in \tilde\calO$ together with~\eqref{eqn:max_margin_stat_xia} and~\eqref{eqn:max_margin_stat_xib} (since $(\cdot, 1) \in \tilde\calO$ is equivalent to $(\cdot) \in \calO$). Since $\nu(s) \geq \|\Gamma(s)\|_2$ and the dual maximizes $-\nu$, the optimal choice is $\nu(s) = \|\Gamma(s)\|_2$, so the dual is equivalent to
\begin{subequations}
\begin{align}
\text{minimize} \quad & \int_{\calS_k} \|\Gamma(s)\|_2\, ds, \\
\text{subject to} \quad & r(s) + \Gamma(s) \in \calO_k.
\end{align}
\end{subequations}
This is a pointwise projection problem: the optimal $\Gamma^\star(s)$ is the displacement from $r(s)$ to its closest point in $\calO_k$, i.e., $\Gamma^\star(s) = r^\star(s) - r(s)$ where $r^\star(s) = \mathrm{proj}_{\calO_k}(r(s))$.

The corresponding optimal primal solution is
\begin{align}
a^\star(s) &= \frac{r^\star(s) - r(s)}{\|r^\star(s) - r(s)\|_2}, \\
b^\star(s) &= -a^{\star\top}(s)\, r^\star(s),
\end{align}
which is the supporting hyperplane of $\calO_k$ at $r^\star(s)$ with normal pointing toward $\calO_k$.

We now verify optimality via the KKT conditions. Stationarity is already encoded in~\eqref{eqn:max_margin_stat_xia}--\eqref{eqn:max_margin_stat_nu}. Dual feasibility of $\Gamma^\star$ holds since $r(s) + \Gamma^\star(s) = r^\star(s) \in \calO_k$. Primal feasibility of $(a^\star, b^\star)$ follows from convexity of $\calO_k$: the projection theorem gives $(x - r^\star(s))^\top(r(s) - r^\star(s)) \leq 0$ for all $x \in \calO_k$, which rearranges to $a^{\star\top}(s)\,x + b^\star(s) \geq 0$, confirming $(a^\star(s), b^\star(s)) \in \calO^\circ_k$. It remains to check complementary slackness on the four dual-variable terms. The two equality-constraint terms ($\nu$ and $\Gamma$) are trivially zero since $z_0^\star = 1$ and $a^\star = z_1^\star$ hold exactly. For the polar-cone term:
\begin{align}
    \Xi_a^{\star\top}(s)\,a^\star(s) + \Xi_b^\star(s)\,b^\star(s) = r^{\star\top}(s)\,a^\star(s) + b^\star(s) = 0,
\end{align}
since $b^\star(s) = -a^{\star\top}(s)\,r^\star(s)$. For the SOC term, we substitute $\Pi_0^\star(s) = -\|\Gamma^\star(s)\|_2$, $\Pi_1^\star(s) = \Gamma^\star(s) = r^\star(s)-r(s)$, $z_0^\star = 1$, and $z_1^\star(s) = a^\star(s)$:
\begin{align}
    &\Pi_0^\star(s)\,z_0^\star + \Pi_1^{\star\top}(s)\,z_1^\star(s)=
    \\ & -\|r^\star(s)-r(s)\|_2 + \frac{(r^\star(s)-r(s))^\top(r^\star(s)-r(s))}{\|r^\star(s)-r(s)\|_2} = 0.
\end{align}
This confirms that the proposed primal-dual pair satisfies all KKT conditions and is therefore optimal.

The finite-dimensional program~\eqref{eqn:max_margin_planes_discrete} solved in practice replaces this pointwise objective with the supremum of the separation value over each segment, optimized over polynomial planes of fixed degree. An integral cost over the segment is equally valid, and the two objectives coincide in the limit of infinitely many independent segments, where the discrete planes recover the closed-form solution above. With finite-degree planes, however, we found the supremum cost to perform substantially better in practice. The discrete program also retains the feasibility property used in our convergence argument: whenever $r$ is collision-free, its optimal separation value is nonpositive, so the returned planes keep $r$ feasible.

\subsection{Modifications and variants of the trust region updates for Fast Path Planning}
\label{app:fpp}

The smooth phase of the Fast Path Planning algorithm~\cite{marcucci2024fast} minimizes the snap of a smooth trajectory traversing a fixed sequence of safe boxes by alternating between a tangent and a projection step. We can simply replace the objective with the sum of the segment times and add the additional derivative constraints without breaking the convexity of either of the two steps. The projection step computes the Bézier control points at fixed segment durations, and the tangent step now solves an SOCP that lowers the total duration under a per-segment trust region $1/(1+\kappa)\le T_j/\bar T_j\le 1+\kappa$. Both FPP variants in~\Cref{tab:village} use the same $70$-box corridor and differ only in how the trust region is initialized and updated, and in when the alternation stops.

\paragraph{Warm-starting the smooth phase}
We initialize the smooth phase with per-segment durations derived from the polygonal path. We first remove sliver segments, the short corners where the shortest path clips a box over a negligible distance, by dropping the clipped box and reconnecting its two neighboring edges through their intersection, which keeps the path inside the corridor. We then merge consecutive collinear segments into maximal straight runs, extending a run as long as its chord stays within the safe boxes. For each run, we solve our convex minimum-time formulation~\eqref{eqn:convex_min_time}, without the obstacle-avoidance constraints, between its endpoints under the velocity, acceleration, jerk, and snap limits, and split the resulting timing at the original vertices to recover the per-segment durations.

\paragraph{Original} The reference schedule contracts the trust region every iteration, $\kappa\leftarrow\kappa/\omega$, starting from $\kappa_0=1$ with $\omega=3$, and stops the first time the tangent step's predicted relative improvement falls below $10^{-2}$. Because $\kappa$ can never re-expand, the alternation stalls at a high-cost duration of $25.78$~s, reached in $1.05$~s of smooth-phase computation.

\paragraph{Modified} We instead let the trust region re-expand on accepted steps, growing it as $\kappa\leftarrow\min(1.7\,\kappa,\,\kappa_{\max})$ with $\kappa_{\max}=2$ and shrinking it as $\kappa\leftarrow\kappa/2$ on rejected steps. Starting from $\kappa_0=0.5$, every step is accepted, so $\kappa$ opens to the cap within a few iterations and stays there. On its own this reaches the optimum but converges slowly, so we add two further changes. First, a backtracking line search on the tangent step: when the linearized step over-shrinks the durations and the subsequent projection is infeasible, we backtrack along the step to its largest feasible fraction, up to four times, instead of discarding it. Second, an earlier termination at the diminishing-returns elbow, defined as $3$ consecutive iterations that improve the duration by less than $10^{-2}$. Together these reach an $11.68$~s trajectory in $20$ iterations and $10.28$~s of smooth-phase computation. Both times exclude the $1.64$~s of box preprocessing shared by the two variants.

\subsection{Nonlinear baseline}\label{app:nonlinear_baselines}

The nonlinear baseline solves the minimum-time problem~\eqref{eq:statement} directly as a nonlinear program, using the same \bez spline parameterization as our approach. We describe the common formulation first and then the collision-avoidance strategy.

\paragraph{Common formulation}
We represent the normalized trajectory $r$ as a \bez spline with $M$ segments, each a degree-$D$ curve with control points $\pi_d^{(m)} \in \mathbb{R}^n$ for $d = 0, \ldots, D$ and $m = 1, \ldots, M$. Unlike our approach, the baseline keeps a single segment duration $T > 0$ rather than $T_I = T^I$, so the total trajectory duration is $MT$.

To enforce the derivative constraints~\eqref{eq:statement_derivative}, we introduce auxiliary normalized velocity control points $p^v_{m,d} \in \mathbb{R}^n$ for $d = 0, \ldots, D-1$, and normalized acceleration control points $p^a_{m,d} \in \mathbb{R}^n$ for $d = 0, \ldots, D-2$. By~\Cref{prop:derivative}, the $d$th control point of the normalized velocity curve on segment $m$ is $D(\pi_{d+1}^{(m)} - \pi_d^{(m)})$. Relating this to the actual velocity $q^{(1)} = r'/T$ gives the bilinear equality
\begin{align}
    D\bigl(\pi_{d+1}^{(m)} - \pi_d^{(m)}\bigr) = T\, p^v_{m,d}, \quad d = 0, \ldots, D-1.
    \label{eqn:baseline_vel_bilinear}
\end{align}
Applying~\Cref{prop:derivative} again to the velocity curve (whose control points are $T p^v_{m,d}$) and relating to the actual acceleration $q^{(2)} = r''/T^2$ gives
\begin{align}
    (D-1)\bigl(p^v_{m,d+1} - p^v_{m,d}\bigr) = T\, p^a_{m,d}, \quad d = 0, \ldots, D-2.
    \label{eqn:baseline_acc_bilinear}
\end{align}
The derivative constraints then become linear: $p^v_{m,d} \in \mathcal{C}_1$ and $p^a_{m,d} \in \mathcal{C}_2$. $C^1$ and $C^2$ continuity between segments is enforced by equating the last velocity (resp.\ acceleration) control point of segment $m$ with the first of segment $m+1$. The baseline therefore solves
\begin{subequations}
\label{eqn:nonlinear_baselines_common}
\begin{align}
\text{minimize} \quad & T \\
\text{subject to} \quad
& T > 0, \\
& \text{boundary conditions~\eqref{eq:statement_boundary_configuration},~\eqref{eq:statement_boundary_derivative},}\\
&~\eqref{eqn:baseline_vel_bilinear},~~\eqref{eqn:baseline_acc_bilinear}, \\
& p^v_{m,d} \in \mathcal{C}_1,~ p^a_{m,d} \in \mathcal{C}_2, \\
& C^1, C^2 \text{ continuity}, \\
& \text{(collision avoidance, see below),}
\end{align}
\end{subequations}
with the nonlinear program solved by the off-the-shelf solver SNOPT~\cite{gill2005snopt}.

\paragraph{Collision avoidance}
We enforce collision avoidance at sampled points along each segment. For each segment $m$ and each of $J$ uniformly-spaced sample parameters $s_j \in [0,1]$, we introduce an auxiliary variable $q_j^{(m)} \in \mathbb{R}^n$ with the linear equality
\begin{align}
    q_j^{(m)} = r_m(s_j) = \sum_{d=0}^{D} \beta_d(s_j)\,\pi_d^{(m)},
\end{align}
and enforce collision avoidance at each sample point via Drake's \texttt{MinimumDistanceLowerBoundConstraint}~\cite{drake}. This constraint evaluates the signed distances between all geometry pairs registered in the scene graph at configuration $q_j^{(m)}$, and enforces that the minimum signed distance is non-negative using a smooth penalty formulation.

\subsection{Bin unloading experiment details}\label{app:pallet_randomization}

\paragraph{Package randomization}
Each random instance places 6 packages (3 red, 3 blue) in the bin. Package dimensions are sampled independently and uniformly from $[8, 10, 3]\,\text{cm}$ to $[12, 12, 13]\,\text{cm}$ (width $\times$ depth $\times$ height). The 2D position of each package in the bin is sampled uniformly within $95\%$ of the available bin width and $90\%$ of the available bin depth. Packages are placed on the bin floor at the appropriate height. Placement is generated by rejection sampling until all packages are mutually separated by at least $5.5\,\text{cm}$.

After placement, packages are sorted by their distance to the bin center and assigned colors (red / blue) by selecting uniformly at random from the subset of red/blue assignments (3 of each) for which all package grasp points lie within the respective arm's reachable workspace, with a $5\,\text{cm}$ inset from the workspace boundary. Instances for which no valid assignment exists are rejected and resampled.

\paragraph{Inverse kinematics}
We obtain an initial joint configuration at each waypoint using the analytical IK solver of~\cite{HeLiu2021}, which provides closed-form solutions for the Franka Emika Panda, and select a collision-free configuration for both arms by discretizing the redundant elbow angle. This initial configuration initializes the differential IK controller that tracks the optimized task space trajectory, described in~\Cref{app:diffik}.

\subsection{Differential IK for task space trajectory tracking}\label{app:diffik}

The planner emits a 6D task space trajectory $r(t) = [r_\mathrm{red}(t), r_\mathrm{blue}(t)] \in \mathbb{R}^6$ specifying the position of each gripper. We convert $r$ to a 14-DoF joint trajectory $q(t) \in \R^{14}$ by integrating two independent per-arm differential inverse-kinematics (IK) controllers. The same procedure is used in simulation and on hardware.

We index arm $a \in \{\mathrm{red}, \mathrm{blue}\}$ and write $p_a(q_a) \inR^3$, $R_a(q_a) \in \mathrm{SO}(3)$ for the world-frame position and rotation of \texttt{panda\_link8}, and $J_a(q_a) \in \R^{6 \times 7}$ for its spatial Jacobian. At time $t$ the commanded spatial velocity is $\xi_a^\mathrm{cmd}(t) = [\omega_a^\mathrm{cmd}(t);\, v_a^\mathrm{cmd}(t)] \inR^6$ with translational feedback that follows the planner reference and rejects accumulated position drift,
\begin{align}
v_a^\mathrm{cmd}(t) &= \dot r_a(t) + k_p\bigl(r_a(t) - p_a(q_a)\bigr),\\
\omega_a^\mathrm{cmd}(t) &= k_o\,\theta\,\hat e,\quad (\hat e, \theta) = \mathrm{axisangle}\!\bigl(R_a^\star R_a(q_a)^{\top}\bigr),
\end{align}
where $R_a^\star = R_a(q_a^\mathrm{init})$ is the fixed orientation of \texttt{panda\_link8} at the arm's initial configuration $q_a^\mathrm{init}$, held constant for the duration of the motion. The joint velocity is then obtained by solving the convex QP
\begin{align}
\dot q_a^\star = \mathop{\text{minimize}}_{\dot q_a \in \R^7} \quad &
    w_\mathrm{damp}\|\dot q_a\|_2^2 + w_\mathrm{reg}\bigl\|\dot q_a - k_q (q_a^\mathrm{ref} - q_a)\bigr\|_2^2 \notag \\
\mathrm{s.t.}\quad & J_a(q_a)\,\dot q_a = \xi_a^\mathrm{cmd}(t).
\label{eqn:diffik_qp}
\end{align}
The equality constraint enforces the spatial velocity command exactly; the first cost damps the joint velocity and the second drags its null-space component toward a virtual spring on a nominal reference $q_a^\mathrm{ref}$ (the home configuration). We use $k_p = 35$, $k_o = 5$, $k_q = 0.1$, $w_\mathrm{damp} = 10^{-3}$, $w_\mathrm{reg} = 10^{-1}$, and step $q_a$ forward with explicit Euler at $\Delta t = 5\,\mathrm{ms}$. The configuration is reset to the live robot state at the start of every motion phase, so position-feedback drift does not accumulate across phases. For hardware execution, the integrated $q(t)$ samples are fit with a cubic shape-preserving spline whose analytic derivatives are streamed to the Franka driver at $400\,\mathrm{Hz}$ (\Cref{app:armdriver}).

\subsection{Low-level arm control}\label{app:armdriver}
Each arm is moved by an independent instance of our open-source Franka driver. The driver runs on a dedicated real-time process and exchanges joint commands and status with the experiment over its own LCM bus, one bus per arm. The joint trajectory $q(t)$ of~\Cref{app:diffik} is sampled and published as a stream of joint-position commands at $400\,\mathrm{Hz}$. Each driver tracks its command stream inside the \texttt{libfranka} Franka Control Interface (FCI) loop at $1\,\mathrm{kHz}$. Internally it runs a joint-space impedance controller. It smooths the commanded position with a first-order low-pass filter of cutoff frequency $70\,\mathrm{Hz}$, finite-differences the result for a velocity feedforward and smooths it with a first-order low-pass filter of cutoff frequency $30\,\mathrm{Hz}$, and applies the torque
\begin{equation}
\begin{aligned}
\tau = {}& -K_p\,(q - q^\mathrm{cmd}) - K_d\,(\dot q - \dot q^\mathrm{cmd}) \\
         & + M(q)\,\ddot q^\mathrm{cmd} + c(q, \dot q),
\end{aligned}
\end{equation}
where $M(q)$ is the arm inertia matrix and $c(q, \dot q)$ the Coriolis term. Note that FCI compensates gravity internally. The fixed diagonal gains are $K_p = \mathrm{diag}(962.5, 1155, 1155, 962.5, 192.5, 385, 96.25)\,\mathrm{Nm/rad}$ and $K_d = \mathrm{diag}(37.5, 50, 37.5, 25, 5, 3.75, 2.5)\,\mathrm{Nm\,s/rad}$. A command-expiry watchdog of $10\,\mathrm{ms}$ zeros the commanded velocity if the stream stalls, and \texttt{libfranka}'s rate limiter caps joint velocity and acceleration. The two arms run as separate driver processes on separate buses and are coordinated only through the shared planned trajectory.

\subsection{Waypoint planner}\label{app:WaypointPlanner}
The waypoint planner is a hand-coded baseline that produces a piecewise-linear path in the $6$D task space and retimes each segment with an independent per-segment minimum-time convex program. We use it both as a standalone baseline (\emph{Waypoint}) and as the warm-start polygon supplied to BMTP and EI+SCS throughout this paper.

\paragraph{Path generation}
For each motion phase the planner emits a coarse sequence of $6$D waypoints encoding a vertical lift above the obstacles, an X-side-step to a ``ferry lane'' just inside the bin edge, a Y-translation toward the target zone, and a final descent to the goal. The held-arm (\emph{forward}) variant collapses the final X- and Y-translations into a single diagonal segment, while the empty-arm (\emph{approach}) variant uses an axis-aligned X-adjust after the Y-translation to align with the brick to be grasped. By construction all but at most one segment per phase is axis-aligned, which keeps the path far from the boundaries of the surrounding box obstacles and gives the downstream optimizers a feasible initialization with margin. An additional Y-clearing waypoint is inserted whenever the two arms start with crossed X-positions and a Y-separation tighter than the gripper width, to avoid an inter-gripper collision while the arms swap columns. In all $50$ phases of our hardware experiments this construction yields $6$ waypoints, i.e., $5$ piecewise-linear segments per phase.

\paragraph{Per-segment minimum-time retiming}
Given the waypoint polygon $\{w_0, \dots, w_M\} \subset \R^6$ we fit a \bez spline by retiming each line segment $\overline{w_m w_{m+1}}$ independently. For each segment we solve the convex per-segment minimum-time program from~\cite[\S VI]{marcucci2025biconvex}, specialized to a single $1$D direction $u_m = (w_{m+1}-w_m)/\|w_{m+1}-w_m\|_2$: the $6$D velocity and acceleration set constraints reduce to tight $1$D directional bounds $v_{\max}^m,\,a_{\max}^m,\,a_{\min}^m$ derived from the velocity and acceleration sets $\mathcal V, \mathcal A$, and the segment duration is recovered through the rotated-cone time-scaling reformulation of~\cite[\S VI]{marcucci2025biconvex}. We additionally constrain the first and second derivatives of the $1$D arc-length control points to vanish at both endpoints of every segment, which makes the spline obtained by stitching segments back-to-back $C^2$ at every waypoint. Each program is solved with Clarabel~\cite{goulart2026clarabel}, and the resulting per-segment \bez curves are concatenated into a single composite curve with a contiguous time domain. We use degree $6$ \bez segments on both the hardware experiments and the warm starts shared with BMTP and EI+SCS.

\subsection{EI + SCS implementation details}\label{app:SCSbaseline}

The EI+SCS baseline~\cite{marcucci2025biconvex} optimizes a minimum-time \bez trajectory through a sequence of convex safe sets. We describe how we construct these sets from the waypoint initialization.

\paragraph{Region generation}
Since all obstacles in our task space formulation are convex (axis-aligned box overlaps), we can solve the edge inflation exactly. For each segment $m$ of the piecewise-linear waypoint path and each obstacle $\calO_k$, we solve the convex program
\begin{align}
    \mathop{\text{minimize}}_{x \in \calO_k,\; z \in \calL_m} \;\; \|x - z\|_2^2,
\end{align}
where $\calL_m = \conv\{v_m, v_{m+1}\}$ is the $m$th line segment, following~\cite[Program~8]{werner2025superfast}. Since both $\calO_k$ and $\calL_m$ are convex, this is a convex program. The optimal solution $(x^\star, z^\star)$ yields a separating hyperplane via the gradient of the distance function:
\begin{align}
    a = \frac{x^\star - z^\star}{\|x^\star - z^\star\|_2}, \quad \calH = \{x \mid a^\top x \leq a^\top x^\star - \delta\},
\end{align}
with safety margin $\delta = 10^{-3}$. We solve all obstacle-segment pairs in parallel using Drake's Solve, producing one polytope per segment by intersecting the resulting halfspaces.

\paragraph{Handling degenerate regions}
The biconvex solver~\cite{marcucci2025biconvex} requires that (i) consecutive waypoints are separated by a non-negligible distance, and (ii) non-adjacent regions are disjoint (Assumption~1 in~\cite{marcucci2025biconvex}). In both settings we first prune segments whose edge length falls below $10^{-4}$. In the simulation experiments, we then resolve any remaining overlap between non-adjacent regions $i$ and $i+2$ by removing the intervening region $i+1$, repeating until no non-adjacent pair overlaps. On hardware this proved insufficient, so we instead separate each overlapping pair $(i, i+2)$ directly: we add to both regions a hyperplane whose normal is the direction of the intervening segment $i+1$, placed at its midpoint, leaving a strip of width $2\times10^{-4}$ that belongs to neither region. This guarantees that non-adjacent regions are disjoint while keeping each region around its own segment.